\documentclass{article}

\usepackage{iclr2027_conference,times}
\usepackage[T1]{fontenc}
\usepackage[utf8]{inputenc}
\usepackage{amsmath,amssymb,amsfonts,amsthm,mathtools}
\usepackage{graphicx}
\usepackage{booktabs}
\usepackage{wrapfig}

\usepackage{etoolbox}

\AtBeginEnvironment{table}{\scriptsize}
\AtBeginEnvironment{table*}{\scriptsize}
\newcommand{\fittable}[1]{%
  \resizebox{\ifdim\width>\columnwidth \columnwidth\else \width\fi}{!}{#1}}
\usepackage{multirow}
\usepackage{array}
\usepackage{tabularx}
\usepackage{makecell}
\usepackage{caption}
\usepackage{xspace}
\usepackage{microtype}
\usepackage{enumitem}

\usepackage{xcolor}
\newcommand{\revcolor}{black}
\newcommand{\rev}[1]{\textcolor{\revcolor}{#1}}
\usepackage[hidelinks]{hyperref}
\usepackage{url}

\newtheorem{proposition}{Proposition}
\newtheorem{lemma}{Lemma}

\theoremstyle{definition}

\newcommand{\name}{RIPE-MambaSpike\xspace}
\newcommand{\RR}{\mathbb{R}}\newcommand{\CC}{\mathbb{C}}
\newcommand{\Real}{\operatorname{Re}}
\newcommand{\sigmoid}{\operatorname{\sigma}}
\newcommand{\diag}{\operatorname{diag}}

\title{\name: \underline{R}esolution-\underline{I}ndependent
       Spiking--State-Space Interfaces\\
       for \underline{P}arameter-\underline{E}fficient \mbox{Event-Based} Vision}

\author{
Md Muhiminul Islam \\
Independent Researcher \\
\texttt{muhiminul@gmail.com}
\And
Shoaib Ahmed Dipu \\
Indiana University \\
\texttt{shdipu@iu.edu}
\And
Sayeed Shafayet Chowdhury \\
Indiana University \\
\texttt{saychow@iu.edu}
}

\iclrfinalcopy

\begin{document}
\maketitle
\fancyhead{}
\fancyhead[R]{\raisebox{-12pt}[0pt][0pt]{\textsc{A preprint - 21 November, 2026}}}% lowered 12pt toward the text; body unchanged
\renewcommand{\headrulewidth}{0pt}
\thispagestyle{plain}% no header on the title page

\begin{abstract}
Spiking–Mamba hybrids reach strong accuracy on event-based vision, but existing designs often require tens of millions of parameters. \rev{In hybrids that flatten the spiking feature map and project it densely into the scan, much} of that cost comes from \rev{this interface} rather than from the hybrid architecture itself. In Mamba-Spike, that single resolution-dependent projection accounts for 33.55M of 36.25M parameters. To that end, we introduce \name (\textbf{R}esolution-\textbf{I}ndependent, \textbf{P}arameter-\textbf{E}fficient), which replaces that projection with a hierarchical multi-resolution bridge of fixed channel width. Its deployed footprint is 0.870M parameters, constant at fixed time steps and widths across a 43× range of input areas. Reparameterized spiking stages, temporal decoupled modulation, and a dynamic convex-hull-bounded dual-stream membrane-potential attention preserve accuracy under this compact design. Result-wise, \name is pareto-optimal on CIFAR10-DVS, N-Caltech101, and DailyDVS-200. Notably, on the 200-class DailyDVS-200, a scaled 8.04M configuration achieves 45.7\% top-1 accuracy, the best reported spiking result on that benchmark, and outperforms prior spiking methods with 3.0--15.1× fewer parameters than dense ANNs. Overall, our findings demonstrate that competitive event-based recognition does not require resolution-dependent parameter growth in the spiking-to-state-space bridge. Code is available at \url{https://github.com/MuhiminOsim/RIPE-MambaSpike}.

\end{abstract}

\section{Introduction}\label{sec:intro}
Spiking Neural Networks (SNNs) match event cameras representationally. Their
temporal, event-driven dynamics align with the sparse asynchronous streams of
Dynamic Vision Sensors. A second, more conditional advantage motivates much of
the field: binary spikes let a layer replace multiply--accumulate with
accumulate-only operations, which costs less energy on supporting
silicon~\citep{davies2018loihi,merolla2014truenorth,horowitz2014energy}. That
advantage is contingent on every convolution operand being binary. Under
standard sequential execution, many otherwise-spiking architectures do not meet
that condition, ours included. On conventional GPU/edge runtimes we therefore
make no energy, operation-count or latency claim (\S\ref{sec:complexity}). An exact
algebraic reparameterization recovers the binary-operand property for
neuromorphic deployment (Proposition~\ref{prop:ac-tdm}).

Two obstacles separate these advantages from deployment. Accuracy is the
familiar one. SNNs have trailed dense ANNs on temporally rich benchmarks such
as DVS-Gesture~\citep{amir2017dvsgesture}, CIFAR10-DVS~\citep{li2017cifar10dvs}
and N-Caltech101~\citep{orchard2015nmnist}. The less examined obstacle is
\emph{footprint}. Whatever the arithmetic, the model must fit on the device
beside the sensor. On embedded platforms the parameter budget usually binds
first, held as it is in on-chip SRAM or modest flash, and it binds absolutely:
a model whose weights do not fit cannot run at all, however its arithmetic is
counted. Event cameras ride on micro-aerial vehicles, robots and
wearables~\citep{gallego2020survey}. Our weight storage is $1.7$--$16.1$\,MB
depending on benchmark, against $153$--$294$\,MB for the models leading the
largest event-action benchmarks, and that gap decides what can sit beside the
sensor at all (\S\ref{app:footprint}). This paper addresses that second axis: \emph{parameter}-efficiency.

\paragraph{Where the parameters actually go.}
A growing line of work closes this gap by hybridizing SNNs with high-capacity
sequence models. Spiking--Mamba architectures~\citep{li2024mambaspike} are the
most compelling, since Mamba's selective state-space
scan~\citep{gu2023mamba} models long event streams in linear time with
content-aware gating. These hybrids reach their accuracy at tens of millions
of parameters, and the reason is narrower, and more fixable, than sheer size.
Mamba-Spike's \rev{official implementation~\citep{mambaspike_code}, in its default configuration,} deploys $36.25$\,M parameters at $128^2$. Of
those, $33.55$\,M ($92\%$) sit in a single layer, the projection carrying
the spiking front-end into the scan, which flattens the feature map and
projects it densely:
\begin{equation}\label{eq:flatten-project}
\texttt{Linear}(C\lfloor H/4\rfloor\lfloor W/4\rfloor, d).
\end{equation}
Its parameter count is \emph{quadratic in sensor resolution and independent of
model capacity}. Measured by direct instantiation
(Table~\ref{tab:resolution-scaling}), it grows $49\times$ from $34{\times}34$
to $224{\times}224$, reaching $97.4\%$ of all parameters. The cost is an artifact of \emph{how the
SNN is joined to the SSM}, not an intrinsic price of the pairing. Spiking
\emph{transformer}
baselines~\citep{zhou2023spikformer,yao2023sdt,lee2025statten} are already flat
at $2.57$\,M across resolutions, so the pathology is specific to
flatten-then-project bridges rather than to attention models in general. \rev{Like Mamba-Spike's,
the interface of the patch-tokenized SpikMamba~\citep{chen2024spikmamba} also
grows with resolution, only more slowly ($2\%\to32\%$ from $34^2$ to $224^2$;
Table~\ref{tab:resolution-scaling}).}

\paragraph{Thesis: the interface, not the capacity.}
That measurement gives the thesis its shape. \rev{In flatten-then-project
Spiking--Mamba hybrids, the} parameter cost is a property of the SNN--SSM
interface, not of the capacity the task demands. An interface built on a fixed token width
$d_{\text{m}}$ removes that cost at no loss of accuracy: our deployed count is
constant at $0.870$\,M across a $43\times$ range of input areas (from $34^2$ to
$224^2$). Accuracy-competitive Spiking--Mamba
modeling therefore does not need a large parameter budget, provided the
architecture is co-designed under a strict \emph{reparameterization-first}
principle: pay the parameter cost at training time only, and recover at
inference a compact single-kernel model with sparse activations and linear
attention. The result is a sub-million-parameter
deployed network, $\sim$$41\times$ smaller than \rev{that} Mamba-Spike \rev{configuration}, that
exceeds it \rev{under an identical protocol} on every shared benchmark. Each multi-branch convolution admits a
closed-form fusion into one $3{\times}3$ kernel, generalizing
RepVGG~\citep{ding2021repvgg} to SNNs with BatchNorm-over-Time and spike
non-linearities. Attention is linear in token count with a non-saturating
feature map. The SNN--Mamba interface carries the front-end into the scan at a
fixed width, so no state-space parameter tracks the sensor.

\paragraph{Contributions.}
\begin{enumerate}[leftmargin=1.4em,itemsep=1pt,topsep=1pt]
\item[\textbf{C1.}] \name, a compact hierarchical Spiking--Mamba hybrid.
Reparameterized spiking backbones~\citep{wang2024rtformer} and spiking
state-space hybrids~\citep{li2024mambaspike} have been explored separately. We
claim their \emph{combination}: a resolution-decoupled interface carrying
bounded membrane-potential attention, co-designed under one
reparameterization-first principle, at under one million deployed parameters
(Table~\ref{tab:vs_mambaspike} sets the two architectures side by side).
\item[\textbf{C2.}] \emph{Temporal Decoupled Modulation (TDM)}, a per-channel
learnable temporal filter with three analytical views (discrete-derivative,
$z$-domain FIR, gradient-decoupling). Its commutation with the following
convolution (Proposition~\ref{prop:ac-tdm}) refactors the pair into two static
kernels over binary operands, a result that is not specific to TDM.
\item[\textbf{C3.}] \emph{Dual-Stream Membrane-Potential Attention (DS-MPA)}, a
linear attention with a non-negative ELU feature map between the
membrane-potential (query/key) and prior-spike (value) streams, provably
bounded within the dynamic convex hull of its values (and in $[0,1]$
for unprojected binary spikes), avoiding the
gate-saturation problem of earlier spike-driven linear attention.
\item[\textbf{C4.}] A \emph{Multi-Resolution Spiking--State-Space Aggregation
(MR-S3A) bridge} that fuses the spiking stages at a \emph{fixed} model width,
so the scan's parameter count is independent of $H{\times}W$, verified by
instantiation across $34^2$ to $224^2$ ($43\times$ in area).
\item[\textbf{C5.}] Across five event benchmarks, \name beats a controlled
Mamba-Spike reproduction wherever the two overlap, at $1.2$--$41\times$ fewer
parameters, and establishes a Pareto-optimal trade-off across the object and
large-scale action benchmarks. DailyDVS-200, one of the largest event action
benchmarks published and the hardest we evaluate, is where just one of the
twelve methods we tabulate clears $50\%$ top-1; there \name is the strongest spiking method by
$8.80$\,pp, within $6.25$\,pp of the state of the art at an order of magnitude
fewer parameters.
\end{enumerate}

\section{Related Work}\label{sec:related}

\paragraph{Vision SNNs.}
Direct-training vision SNNs split into spike-driven
transformers~\citep{zhou2023spikformer,yao2023sdt,yao2024sdtv2} and spiking
residual networks~\citep{fang2021sgseu,hu2024msresnet} with learnable membrane
constants~\citep{fang2021plif}. TET loss~\citep{deng2022tet} and
RecDis~\citep{guo2022recdis} are near-universal training tools, orthogonal to
our architecture and adopted in our recipe. The strongest current results on
our benchmarks come from this transformer family. QKFormer~\citep{zhou2024qkformer}
reaches $84.0\%$ on CIFAR10-DVS at $1.50$\,M, STAtten~\citep{lee2025statten}
adds spatial--temporal attention, and the attention-free Spiking Wavelet
Transformer~\citep{fang2024swformer} holds the best N-Caltech101 result we
know of at $88.45\%$, setting the accuracy ceiling. \name does not surpass
them, and instead competes on the accuracy--parameter trade-off, since none
provides an inference-time-fused reparameterized backbone and so each
carries its full multi-branch parameter cost into deployment.
SpikePool~\citep{lee2025spikepool} replaces self-attention with max-pooling
($80.20$/$81.62\%$ at $0.55$\,M, $82.70$/$85.01\%$ at $2.19$\,M on
CIFAR10-DVS/N-Caltech101, Table~\ref{tab:sota-all}); \name beats both on
N-Caltech101 at up to $2.5\times$ fewer parameters, trading $1.3$\,pp for a
$2.5\times$ smaller footprint on CIFAR10-DVS.

\paragraph{State-space models and Spiking--Mamba hybrids.}
The SSM lineage~\citep{gu2022s4,gu2023mamba} offers linear-time sequence
modeling with content-aware gating, adapted to vision by
Vim~\citep{zhu2024vim} and VMamba~\citep{liu2024vmamba}.
\textbf{Mamba-Spike}~\citep{li2024mambaspike}, our primary baseline, couples a
small spiking encoder to a Mamba stack. SpikMamba~\citep{chen2024spikmamba}
adds window-based spiking linear attention tuned for fixed-resolution
gesture sequences ($99.01\%$ on DVS-Gesture at $0.18$\,M,
Table~\ref{tab:sota-all}); \name instead keeps its parameter count invariant
across input scales via MR-S3A's fixed-width bridge. Vision
SmolMamba~\citep{bai2026smolmamba} prunes tokens inside a spiking SSM. Its
target is token count; ours is the \emph{parameter} coupling between
resolution and the interface, a different and complementary lever. We prune
nothing (every MR-S3A token enters the scan), and our invariance follows
instead from the bridge's fixed width and the length-independent
$\mathcal{O}(d_{\text{m}}^2)$ cost of Mamba blocks. A system could adopt both.
FLAMES~\citep{chakraborty2025flames} takes a complementary event-native SSM
route, adapting HiPPO to inter-spike intervals rather than touching the
interface. SPARTA~\citep{jang2025sparta} prioritizes tokens by spike timing,
reaching $98.78/83.06\%$ on DVS-Gesture/CIFAR10-DVS at ${\sim}13.8$\,M,
$15\times$ our footprint. TP-Spikformer~\citep{tpspikformer2026} prunes
progressively. TS-SNN~\citep{yu2025tssnn} is a FLOP-free temporal-shift
operator in TDM's spirit; it holds an accumulate-only backbone by
construction, where TDM recovers one by an exact commutation
(\S\ref{sec:ac-inference}). \name occupies the sub-million-parameter
corner.

\paragraph{Structural reparameterization.}
RepVGG~\citep{ding2021repvgg}, MobileOne~\citep{vasu2023mobileone} and
RepGhost~\citep{chen2022repghost} fuse multi-branch training graphs into
single kernels. Closest is RTFormer~\citep{wang2024rtformer}, which
reparameterizes a spiking block and folds temporal BN into the neuron
\emph{threshold}. We instead retain BNTT as a per-timestep affine absorbed
into a final scale--shift, target the sub-million regime, and pair
reparameterization with a state-space rather than an attention backbone. Our
claim is not that reparameterized spiking backbones are novel; we position \name
as the first to combine one with a hierarchical Spiking--Mamba interface at
this scale (Appendix~\ref{app:method}). DS-MPA fixes the normalizer-collapse
issue of linear attention~\citep{katharopoulos2020linear,choromanski2021performer} by
construction (\S\ref{sec:dsmpa}).

\section{Preliminaries}\label{sec:prelim}
We accumulate a DVS event stream into $T$ time bins to form
$\mathbf{X}\in\RR^{B\times T\times 2\times H\times W}$ (two polarity channels).
A leaky integrate-and-fire (LIF) neuron integrates input $x_t$ into a membrane
potential $u_t=\beta(u_{t-1}-v_{\text{th}}s_{t-1})+x_t$ and fires $s_t=\Theta(u_t-v_{\text{th}})$ with a
soft reset. Because $\Theta'$ is a Dirac delta, we backpropagate through the ATan
surrogate~\citep{neftci2019surrogate}, whose sharpness $\alpha$ we anneal from
$2$ to $4$, or to $3$ on CIFAR10-DVS and N-Caltech101 (Appendix~\ref{app:hyper}). That single $\alpha$ is also the temperature of the soft-spike
sigmoid in the consistency loss (Section~\ref{sec:loss}), and is distinct from
the per-channel TDM coefficients $\lambda$ and the scalar loss weights. We use
per-timestep BatchNorm (BNTT)~\citep{kim2021bntt}, which keeps $T$ independent
affine pairs. Mamba~\citep{gu2023mamba} evolves a hidden state via a selective,
input-dependent state-space recurrence at $\mathcal{O}(Ld_{\text{state}}d)$ cost (full
discretization in Appendix~\ref{app:method}). We use the standard
synaptic-operation (SOP) accounting convention~\citep{horowitz2014energy}: a
spiking convolution with average
pre-synaptic spike rate $\bar r$ contributes $\bar r\!\cdot\!\Phi$
accumulate-only (AC) operations \emph{only when its operand is binary}, and
dense multiply-accumulate (MAC) operations otherwise. In \S\ref{sec:complexity}
we use this convention to state why, under standard sequential execution, the
deployed backbone does not qualify for accumulate-only accounting, and we
report picojoule figures only for the dual-kernel reparameterization
(Proposition~\ref{prop:ac-tdm}) that recovers it for neuromorphic deployment.

\section{Method: \name}\label{sec:method}

\name maps a 5-D event tensor to per-timestep class logits (averaged over $T$ at
evaluation, supervised per-timestep by TET) through: an entry BNTT; three
\emph{RepSpikeStage} blocks of channels $[32,64,128]$ with stride-2 downsampling
and heterogeneous neurons (CSiLIF, SiLIF, SiLIF); a SpikeToRate$+$MR-S3A bridge
into a $d{=}32$ token space widened to $128$; a
three-level hierarchical bidirectional Mamba backbone; and a linear head
(Fig.~\ref{fig:arch}). DS-MPA is embedded inside each stage.

\begin{figure*}[t]
\centering
\includegraphics[width=\linewidth]{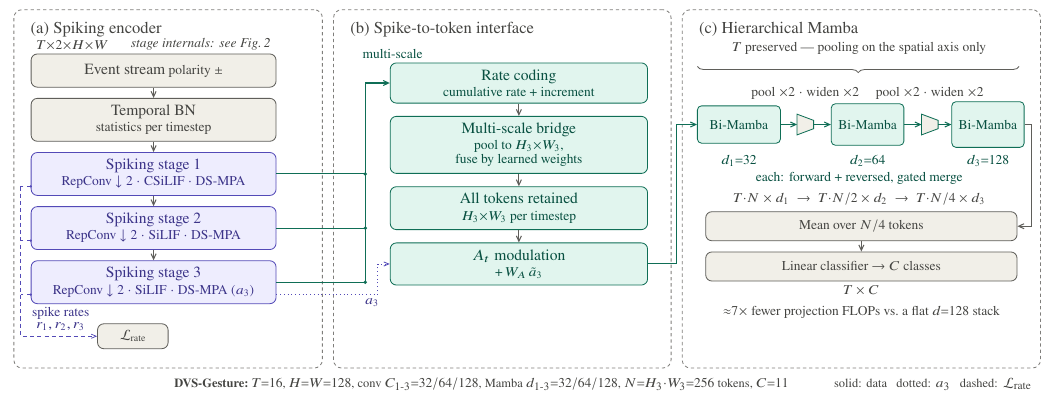}
\caption{Overview of \name (DVS-Gesture). \textbf{(a)} Three RepSpikeStage blocks
(TDM $+$ reparameterizable convolutions $+$ DS-MPA, stride-2) emit the final
attention map $a_3$ and firing rates. \textbf{(b)} Multi-scale rates/increments
are fused (MR-S3A) into a fixed-width token sequence.
\textbf{(c)} Three gated bidirectional Mamba blocks process the tokens, each
transition pooling the spatial axis $\times 2$ and doubling width.}
\label{fig:arch}
\end{figure*}

\subsection{Reparameterized Spiking Stage}\label{sec:repspike}
At training time each backbone convolution is a sum of seven parallel branches
(3$\times$3, 1$\times$1, 1$\times$3, 3$\times$1, 1$\times$1$\to$3$\times$3,
identity, avg-pool), each with its own BatchNorm (Fig.~\ref{fig:stage}b). This
sum fuses exactly to one $3{\times}3$ kernel at inference.

\begin{proposition}[Branch fusion]\label{prop:fusion}
Let $s_b=\gamma_b/\sqrt{\sigma_b^2+\varepsilon}$ be each branch's BN scale.
With $W^\star=\sum_b\mathrm{Pad}_{3\times3}(s_bW_b)$ and
$b^\star=\sum_b(\beta_b-\mu_b s_b)$, we have
$\mathrm{RepConv}(x)\equiv W^\star * x + b^\star$ for all $x$.
\end{proposition}
Proof and the $1{\times}1{\to}3{\times}3$ composition are in
Appendix~\ref{app:proofs}. The seven branches have different kernel areas, so the
saving is $26/9\approx2.9\times$ per unit. Branch
BNs are standard 2D normalizations across the merged $B{\times}T$ batch and
absorb completely into the single static $3{\times}3$ kernel $W^\star$ and bias
$b^\star$; BNTT sits downstream of the sum as a per-timestep affine vector
($2Tc_{\text{out}}$ scalars), so fusion does not replicate weights across time.
Each stage chains two TDM$\to$RepConv$\to$BNTT$\to$neuron units with DS-MPA
bridging the two neurons:
\begin{align}
y_1&=\mathrm{BNTT}(\mathrm{RepConv}_1(\mathrm{TDM}_{\lambda_1}(x))),\quad(S_1,U_1)=\mathcal{N}_1(y_1),\\
y_2&=\mathrm{BNTT}(\mathrm{RepConv}_2(\mathrm{TDM}_{\lambda_2}(S_1))),\quad
\hat y_2=\mathrm{DS\text{-}MPA}(U_1,S_1^{\text{prev}},y_2),\ (S_2,U_2)=\mathcal{N}_2(\hat y_2).
\end{align}

\subsection{Temporal Decoupled Modulation (TDM)}\label{sec:tdm}
For a per-channel $\lambda\in\RR^C$, TDM is the first-order filter
$\mathrm{TDM}_\lambda(x)_{t}=(1+\lambda_c)x_t-\lambda_c x_{t-1}$ (with
$x_0{:=}0$), adding one scalar per input channel per insertion ($322$ total, $0.04\%$ of the
model). It admits three views, with full forms in Appendix~\ref{app:method};
we state the two load-bearing ones here. As a learnable discrete derivative
$\tilde x_t=x_t+\lambda_c(x_t-x_{t-1})$, it is exactly the identity at
$\lambda{=}0$, so zero-initialization makes TDM inert at iteration zero and
training starts from the TDM-free network rather than a perturbed one. As a
$z$-domain filter $H_c(z)=(1+\lambda_c)-\lambda_c z^{-1}$, each channel learns
its own bandwidth: $|H_c(e^{j\omega})|^2=1+2\lambda_c(1+\lambda_c)(1-\cos\omega)$
is high-pass for $\lambda_c\!>\!0$ (Prop.~\ref{prop:tdm-spec},
Appendix~\ref{app:method}; proof in Appendix~\ref{app:proofs}). The deployed model learns exactly this depth
profile. The first stage's coefficient reaches $\lambda\approx0.72$, high-pass
and motion-emphasizing, while the deepest stage settles near $\lambda\approx0$,
close to identity, consistent with early-pathway motion sensitivity. TDM
also \rev{gives each layer's input} a surrogate-free temporal gradient path\rev{, local
to that layer} (Appendix~\ref{app:proofs}).

\paragraph{Heterogeneous neurons.}
Stage 1 uses CSiLIF, a complex-valued sigmoid LIF whose per-channel pole
$\rho_c=-e^{\ell_c}+j\omega_c\in\CC^-$ turns a real LIF's low-pass response
into a damped oscillation, a learnable band-pass channel. The form
$-e^{\ell_c}$ keeps it stable for every $\ell_c$ without clipping (readout
$U=2\Real(m)$; Appendix~\ref{app:method}). Stages 2--3 use the real-valued
SiLIF. Neither neuron is our contribution: SiLIF follows
PLIF~\citep{fang2021plif}, CSiLIF the resonate-and-fire
idea~\citep{izhikevich2001resonate}. Our finding is that the assignment
CSiLIF--SiLIF--SiLIF attains the highest accuracy of the eight sweep
configurations (Appendix~\ref{app:exp}), with resonance helping most at the
high-resolution first stage.

\subsection{Dual-Stream Membrane-Potential Attention}\label{sec:dsmpa}
DS-MPA is a per-timestep linear attention with queries/keys from the current
membrane $U_1$ and values from the previous spike tensor $S_1^{\text{prev}}$.
With the non-negative feature map $\phi(z)=\mathrm{ELU}(z)+1$,
\begin{align}
A&=\frac{\phi(Q)\,(\phi(K)^{\!\top}V)}{\phi(Q)\,(\phi(K)^{\!\top}\mathbf1)+\varepsilon}\in\RR^{N\times C},\label{eq:dsmpa-attn}\\
\hat X &= X+\gamma\odot(A\odot X),\label{eq:dsmpa-mod}
\end{align}
with a per-channel gate $\gamma$ (init $1$). Dimensionally, the $N{=}HW$
spatial positions are the sequence axis and the $C$ channels the feature axis.
The key--value contraction $\phi(K)^{\!\top}V\in\RR^{C\times C}$ is over the
sequence axis at cost $\mathcal{O}(NC^2)$, so DS-MPA mixes \emph{channels}
conditioned on membrane content rather than tokens, and the normalizer keeps
$A$ bounded. Width does not inflate that $O(NC^2)$ term here, since $N$
shrinks $16\times$ as $C$ grows $4\times$. Isolated, DS-MPA is $7.9\%$ of a
step, $4.3\%$ of parameters and $7.4\%$ of the deployed count; a $3.3\times$
cheaper gate recovers $5.5\%$ of step time and a $4\times$ low-rank cut
(which provably retains the bound below) only $0.875\!\to\!0.827$\,M, both
against the $3.16$\,pp DS-MPA is worth on N-Caltech101, so we deploy the full
operator (Tables~\ref{tab:dsmpa-cost} and~\ref{tab:q6-lowrank}, the latter
untrained; pseudo-code in Fig.~\ref{fig:dsmpa-pseudo}).
\begin{proposition}[Boundedness]\label{prop:bounded}
If $\phi\ge0$ element-wise, then $A_{i,\cdot}$ is a convex combination of the
\emph{rows of} $V$ contracted by $\eta_i=d_i/(d_i{+}\varepsilon)\in[0,1)$, so
$\min(0,\min_j V_{j,c}) \le A_{i,c} \le \max(0,\max_j V_{j,c})$; unprojected
binary $V$ gives $0\le A_{i,c}\le1$.
\end{proposition}
The output can therefore neither diverge nor be driven to a constant by a
saturating scalar; there is none, and the bound is \emph{dynamic} rather than
a fixed $[0,1]$ interval (proof in Appendix~\ref{app:proofs}). It holds for any $V$ and any
parameterization of $W_Q,W_K,W_V$, so the reduced variants of
Table~\ref{tab:q6-lowrank} inherit it. We do not constrain $W_V$ to make $V$ non-negative. \rev{A parameter-matched sigmoid gate, $\sigma(QK^\top\!/\sqrt{C})V/N$, is $1.57$\,pp less accurate on N-Caltech101 (three seeds, $p{<}10^{-3}$; Appendix~\ref{app:method}), so the advantage of the normalized non-negative kernel is measured, not only argued.} DS-MPA costs $\mathcal{O}(NC^2)$ time against $\mathcal{O}(N^2C)$ for softmax attention; learned attention maps and per-class accuracy are in Appendix~\ref{app:attn}.

\subsection{MR-S3A Bridge}\label{sec:mrs3a}
Each stage's spikes are encoded as a 2-channel cumulative-rate$+$first-difference
summary, pooled to the deepest resolution and fused by softmax-weighted
multi-scale projection into the bridge token space. The bridge width
$d_{\text{m}}$ is a \emph{fixed} hyperparameter, independent of $H{\times}W$.
This is what decouples the deployed parameter count from resolution, since the
Mamba blocks downstream are $\mathcal{O}(d_{\text{m}}^2)$ and independent of
sequence length. All $T H_3 W_3$ tokens enter the scan. We apply no token
pruning, and the parameter invariance holds without it
(Table~\ref{tab:resolution-scaling}). A learnable projection additionally
injects the stage-3 DS-MPA attention map into the token sequence
(Fig.~\ref{fig:mrs3a-pseudo}, Appendix~\ref{app:method}).

\subsection{Training Objective}\label{sec:loss}
The loss is $\mathcal{L}=\mathcal{L}_{\text{TET}}+\lambda_{\text{SGC}}\mathcal{L}_{\text{SGC}}+\lambda_{\text{L1}}\bar r$.
Beyond TET~\citep{deng2022tet}, a Surrogate-Gradient Consistency term
$\mathcal{L}_{\text{SGC}}=\mathrm{MSE}(z^{\text{disc}},z^{\text{cont}})$
penalizes the gap between the hard-spike pass and a soft-spike pass
($\Theta\!\to\!\sigmoid(\alpha\cdot)$) through shared weights (symmetric, not
distillation; dual-branch schematic in Appendix~\ref{app:method}), and a small
spike-rate L1 ($\lambda_{\text{L1}}{=}10^{-4}$) holds a depth-dependent operating
point $r_\ell\in[0.015,0.19]$. The full composite loss is used for
DVS-Gesture and N-MNIST; CIFAR10-DVS and N-Caltech101 use a TET-only
objective ($\lambda_{\text{SGC}}{=}\lambda_{\text{L1}}{=}0$), and
DailyDVS-200 disables SGC while retaining the rate penalty
(Table~\ref{tab:hyper}).

\section{Parameter and Deployment Cost}\label{sec:complexity}

\paragraph{Parameters.}
Each RepConv fuses from $26c_{\text{in}}c_{\text{out}}$ training weights to
$9c_{\text{in}}c_{\text{out}}$ at inference (Prop.~\ref{prop:fusion},
derivation in Appendix~\ref{app:method}), a $26/9 \approx 2.89\times$ reduction
per convolution unit. Across the full backbone that is $2.76\times$,
$1.014\to0.368$\,M, not a naive $7\times$. Fusion leaves the hierarchical Mamba, bridge and
head untouched. The hierarchy is a parameter-economy device, not an
operation-count one: a bidirectional block costs $\mathcal{O}(d^2)$, so a flat
three-block stack at $d{=}128$ would cost $2.21\times$ the hierarchy's
$361{,}760$ ($1.31$ vs.\ $0.875$\,M deployed), though at a \emph{matched}
budget it is not faster.
We quote two deployed counts: $869{,}593\approx\textbf{0.870\,M}$ with a
10-class head at $T{=}10$, and $875{,}122\approx\textbf{0.875\,M}$ for
DVS-Gesture's 11-class head at $T{=}16$, down from $1.522$\,M at training
time by branch fusion (Table~\ref{tab:param-fusion}). They differ only in $T$
and class count, neither of which depends on $H{\times}W$; the
head-independent core ($868{,}303$) and the full reconciliation are in
Appendix~\ref{app:method}. At fixed $T$ and widths no parameter scales with
$H{\times}W$: the Mamba blocks are $O(d^2)$ and independent of sequence length,
which we verify by direct instantiation from $34^2$ through $224^2$
($43\times$ range in pixel area).
Mamba-Spike instead flattens the grid into one dense
$\texttt{Linear}(C\lfloor H/4\rfloor\lfloor W/4\rfloor, d)$, exactly quadratic
in $\lfloor H/4\rfloor$ and independent of capacity. Fig.~\ref{fig:interface}
plots both regimes.

\begin{figure}[t]
\centering
\includegraphics[width=\linewidth]{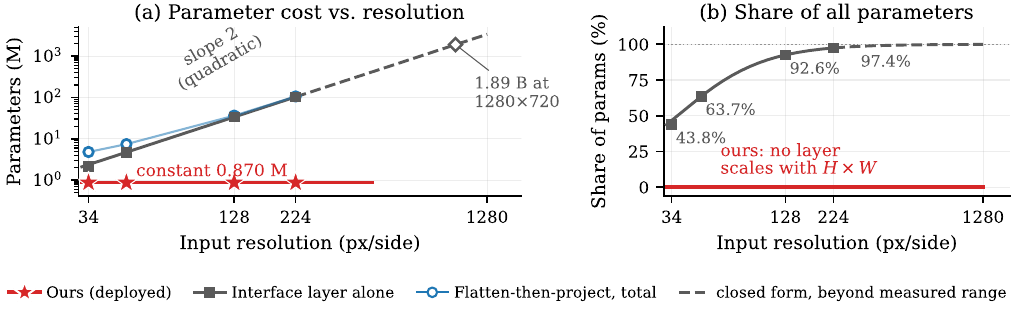}
\caption{\textbf{The SNN$\to$SSM interface's parameter cost is quadratic in
sensor resolution; \name's is constant.} Markers are \emph{instantiated and
counted} (Appendix~\ref{app:ressweep}); the curve is
Eq.~\eqref{eq:flatten-project}, which reproduces all four exactly, dashed where
extrapolated. At the $1280{\times}720$ of an HD event sensor that one layer
reaches $1.89$\,B parameters (${\approx}7$\,GiB \texttt{fp32}); \name is tested flat
at $0.870$\,M from $34^2$ to $224^2$.}
\label{fig:interface}
\end{figure}

Constant parameters do not mean constant cost. Over the same range peak
activation memory grows ${\sim}30\times$ and throughput falls
(Table~\ref{tab:q7-resolution-cost}), so weight storage is decoupled from
resolution but total memory is not. Measured FLOP/MACs and their $k,T,C$
scaling are in Table~\ref{tab:flops}: $4.50$\,GMACs deployed, and $8\times$ the
token budget costs only $11\%$ more compute.

\paragraph{Accumulate-only inference by linear commutation.}\label{sec:ac-inference}
A linear temporal filter commutes with the convolution that follows it, and the
pair refactors exactly into two static kernels over the filter's
\emph{original} operands, which for TDM are binary spikes. Run sequentially
TDM emits $\{-\lambda_c,0,1,1+\lambda_c\}$ and a GPU evaluates dense MACs
($4.50$\,GMACs deployed); a neuromorphic target evaluates this instead:
\begin{proposition}[Dual-Kernel AC Inference]\label{prop:ac-tdm}
Let $W^\star\in\RR^{C_{\text{out}}\times C_{\text{in}}\times 3\times 3}$ and $b^\star\in\RR^{C_{\text{out}}}$ be the fused RepConv parameters, and $\lambda\in\RR^{C_{\text{in}}}$ be the TDM parameter. With scaled static kernels $W_a^\star = W^\star \cdot \diag(1+\lambda)$ and $W_b^\star = W^\star \cdot \diag(\lambda)$, the composition exactly satisfies:
\begin{equation}\label{eq:ac-commutation}
\mathrm{Conv}(W^\star, \mathrm{TDM}_\lambda(x))_t \;=\; W_a^\star * x_t - W_b^\star * x_{t-1} + b^\star.
\end{equation}
\end{proposition}
The identity applies exactly when $x_t,x_{t-1}$ are \emph{original binary
spikes}, where Eq.~\eqref{eq:ac-commutation} executes as \textbf{two
accumulate-only (AC) convolutions} plus an element-wise subtraction. It holds
for five of six backbone convolutions, failing only for the first, whose
operand is the binned event tensor after entry normalization and is
real-valued; that one stays MAC, and is $1.54\%$ of backbone convolution
operations at $128^2$, leaving $98.5\%$ AC-eligible. DS-MPA, the MR-S3A projections, the complex CSiLIF state and the
selective scan remain dense\rev{; Appendix~\ref{app:energy} costs the whole model}. At 45\,nm ($E_{\mathrm{MAC}}{=}4.6$\,pJ, $E_{\mathrm{AC}}{=}0.9$\,pJ)
the two AC convolutions cost $1.8\bar r$\,pJ: at most $1.8$\,pJ, and
$0.09$--$0.36$\,pJ over the measured $\bar r\in[0.05,0.20]$. That counts
synaptic operations for those layers only, against a dense MAC at $100\%$
activity, and excludes memory access. Proof in Appendix~\ref{app:proofs}; \rev{on SpiNNaker hardware, a layer executed this way matches PyTorch to $r{=}0.9999998$ (Appendix~\ref{app:energy})}.

\paragraph{Latency and memory.} Fusing the multi-branch graph cuts $B{=}8$
latency $1.53\times$. At $B{=}1$, the edge operating point, \name's footprint
is below the baseline's; at $B{=}8$ its activations are ${\sim}3.8\times$
larger, so it is slower and larger in total memory: the trade is weight
memory for activation memory (Appendix~\ref{app:latency}).

\paragraph{\rev{Whole-model cost and storage.}}
\rev{The figures above cost the AC backbone only. NeuroBench~\citep{yik2025neurobench},
run on the $264$ DVS-Gesture test clips against a trained Mamba-Spike under
identical metrics (Table~\ref{tab:energy-whole}), counts the whole deployed graph:
$1.16$\,M weights including the previous-step kernels $W_b^\star$ of Prop.~\ref{prop:ac-tdm} ($4.7$\,MB in fp32) against $36.25$\,M
($145$\,MB), at $89.4\%$ against $98.5\%$ activation sparsity. Our compute is
dominated by the dense components (DS-MPA, bridge and Mamba $89\%$, spiking
backbone $10\%$), and Mamba-Spike's very sparse front-end computes less per
inference ($0.27$ vs.\ $6.25$\,mJ estimated), so we make no per-inference
energy claim over it. The advantage is storage. At $8$ bits our weights
($1.2$\,MB) fit the on-chip weight memory of a single Loihi~2 or SpiNNaker~2
chip, whereas Mamba-Spike's need two to three chips at $128^2$ and at least
$119$ at HD resolution (Table~\ref{tab:chip-fit}). \name thus meets the
prerequisite for single-chip deployment on every processor listed; Mamba-Spike
meets it on none.}

\section{Experiments}\label{sec:exp}

\paragraph{Setup.}
We evaluate on DVS-Gesture ($T{=}16$, $128^2$),
CIFAR10-DVS ($T{=}10$, $128^2$), N-Caltech101 ($T{=}10$, $128^2$, 101 classes),
DailyDVS-200~\citep{dailydvs200} ($T{=}10$, $224^2$, 200 action classes, official
subject-disjoint split, 14{,}773/3{,}180/4{,}093 clips) and
N-MNIST~\citep{orchard2015nmnist} ($T{=}10$, $34^2$), using each dataset's
standard split. All numbers are
single-pass top-1 (no test-time augmentation) after clean batch-norm
recalibration. The identical
protocol, schedule and split are used for the headline results and all ablations,
so every internal comparison is like-for-like. Every headline result is a five-seed measurement: DVS-Gesture
$98.18\pm0.31$ (best $98.48$), CIFAR10-DVS $81.28\pm0.11$, N-Caltech101
$85.49\pm0.18$, N-MNIST $99.55\pm0.02$, DailyDVS-200 $45.74\pm0.40$. Tables quote
the best run, matching how the published baselines report; $\sigma$ is used in
Section~\ref{sec:ablation} to say which ablation effects clear noise. Per-seed
values and the per-dataset $\sigma$-in-samples conversions are in
Appendix~\ref{app:seeds}. Full hyperparameters are in Appendix~\ref{app:exp}; we implement \name in PyTorch with the official
\texttt{mamba\_ssm} kernels, training on an NVIDIA A100 GPU. Training code is public at \url{https://github.com/MuhiminOsim/RIPE-MambaSpike}.

\paragraph{Baselines.}
Mamba-Spike~\citep{li2024mambaspike} is the primary baseline, reported three
ways. The first is its published figures, which are for a larger configuration
with no parameter count stated. The other two run the \emph{\rev{official}} code~\citep{mambaspike_code},
under its own protocol ($94.32$/$48.90$/$99.36\%$) and under ours
($95.83$/$65.30$/$99.49\%$). Our
protocol \emph{improves} it everywhere, by $16.40$\,pp on CIFAR10-DVS. All quoted
Mamba-Spike parameters are measured from that code ($36.25$\,M at $128^2$) and
do not pair with the published accuracies, as detailed in
Appendix~\ref{app:proto}. Secondary baselines (SpikFormer, SDT, MS-ResNet,
QKFormer, SWformer, STAtten, RTFormer, TET-VGG) are quoted from their papers
and marked $^\ast$ (cross-protocol).

\paragraph{Results.}
Table~\ref{tab:sota-all} reports all five benchmarks, and three patterns hold
across the four with published parameter counts. \name exceeds the
Mamba-Spike reproduction on every shared benchmark by $2.65$/$16.10$/$0.08$\,pp
(best-seed). The parameter ratio is benchmark-dependent: $41\times$ on
DVS-Gesture and CIFAR10-DVS ($36.25$\,M), but $1.2\times$ on N-MNIST, whose
$34^2$ input shrinks the baseline to $1.03$\,M. On CIFAR10-DVS,
N-Caltech101 and DailyDVS-200, \name is Pareto-optimal among published models
with reported parameter counts (Fig.~\ref{fig:pareto}). On DVS-Gesture, the
action-specific SpikMamba~\citep{chen2024spikmamba} reaches $99.01\%$ at
$0.18$\,M; \name achieves $98.48\%$ ($98.18\pm0.31\%$ mean, trailing by
${\sim}1$--$2$ clips on this 264-clip test set) while maintaining resolution-independent
parameters across both object and action datasets. Widening the stages shows the
deployed point is an operating choice, not a ceiling. On CIFAR10-DVS,
$[64,128,256]$ ($3.35$\,M) gives $82.80$ and $[128,256,512]$ ($13.14$\,M)
gives $83.40$, beating SPARTA ($13.80$\,M, $83.06$) but not QKFormer's
$84.00$ at $1.50$\,M. It leads on accuracy nowhere: SDT ($99.30$), QKFormer ($84.00$) and SWformer
($88.45$) remain ahead, so the claim is frontier position, not dominance. On
CIFAR10-DVS the published Mamba-Spike figure ($92.5$) far exceeds our
reproduction, for the protocol reasons of Appendix~\ref{app:proto}\rev{; against
it \name is smaller, not more accurate}.
DailyDVS-200, at 200 classes and $224^2$ the largest scale we test, is nowhere
near solved: only one of the twelve entries in Table~\ref{tab:sota-all} clears
$50\%$ top-1, their median is $41.95$ and the ceiling $51.99$. \name reaches $45.74\pm0.40$ top-1 and $69.50\pm0.49$ top-5 at
$8.04$\,M: the best spiking method by $8.80$\,pp, fourth of twelve overall
($3.79$\,pp above that median), and Pareto-optimal among entries with a
published count, at $3.5\times$ smaller than Swin-T, $9.6\times$ than EvMamba
and $15.1\times$ than TimeSformer, which it beats outright. Width is a
capacity requirement at this scale, not an operating choice: the lite/wide
pair gives $42.41\pm0.34$ at $3.82$\,M against $45.74$ at $8.04$\,M
($\Delta{=}3.33$\,pp, Welch's $t{=}14.1$). Further per-benchmark detail,
including the top-5 asymmetry against TSM, is in Appendix~\ref{app:seeds}.

% Two-column, Backbone column dropped. Rationale: single-column at 44 rows was
% ~70% of a page, which meant LaTeX could not place it alongside the other §6-§7
% floats and it drifted past the references (measured: p21 of a 26-page build).
% Backbone is the widest column and the least load-bearing -- method names imply
% it -- and dropping it is what makes a two-column split fit without the
% minipages colliding. The master tree retains the full five-column version.
\begin{table}[t]
\centering
\caption{State-of-the-art comparison across all five event benchmarks
(inference-time parameters). $^\ast$ published; $^\ddagger$/$^\dagger$ the \rev{official}
Mamba-Spike code\rev{~\cite{mambaspike_code}} under its own / our protocol; $^\S$ no parameter count is
attributable (Appendix~\ref{app:proto})\rev{; $^\P$ SDT reports no N-Caltech101 result, value as
published by~\cite{lee2025statten,lee2025spikepool}. STAtten rows use its SDT-2-256 backbone}. Among general-purpose backbones, \name
is smallest in three of four and Pareto-optimal across datasets
(action-specific SpikMamba~\cite{chen2024spikmamba} is smaller at $0.18$\,M
on DVS-Gesture), leading nowhere on accuracy: it concedes at most $2.9$\,pp
to models $1.7$--$41\times$ its size, $6.25$\,pp to the best DailyDVS-200
result. DailyDVS-200 baselines are frame-based video models with no spiking time-step, so their $T$ column reads ``--''. Best per block \underline{underlined}.}
\label{tab:sota-all}
\renewcommand{\arraystretch}{0.79}
\setlength{\tabcolsep}{3pt}
\scriptsize
\begin{minipage}[t]{0.48\linewidth}\centering
\begin{tabular}{lccc}
\toprule
Method & $T$ & Params & Acc.\ (\%) \\
\midrule
\multicolumn{4}{l}{\emph{DVS-Gesture} ($T{=}16$)} \\
TET~\cite{deng2022tet}                   & 16 & 3.66\,M  & 97.92$^\ast$ \\
SpikFormer~\cite{zhou2023spikformer}     & 16 & 2.57\,M  & 98.30$^\ast$ \\
SDT~\cite{yao2023sdt}                    & 16 & 2.57\,M  & \underline{99.30}$^\ast$ \\
MS-ResNet~\cite{hu2024msresnet}          & 16 & 11.20\,M & 98.61$^\ast$ \\
RTFormer~\cite{wang2024rtformer}         & 16 & ---      & 98.61$^\ast$ \\
SPARTA~\cite{jang2025sparta}             & 20 & 13.80\,M & 98.78$^\ast$ \\
SpikMamba~\cite{chen2024spikmamba}       & 16 & 0.18\,M  & 99.01$^\ast$ \\
Mamba-Spike~\cite{li2024mambaspike}$^\ast$ & 16 & ---$^\S$ & 97.80 \\
Mamba-Spike~\cite{li2024mambaspike}$^\ddagger$ & 16 & 36.25\,M & 94.32 \\
Mamba-Spike~\cite{li2024mambaspike}$^\dagger$ & 16 & 36.25\,M & 95.83 \\
\textbf{\name (ours)}                    & 16 & \textbf{0.88\,M} & \textbf{98.48} \\
\midrule
\multicolumn{4}{l}{\emph{CIFAR10-DVS} ($T{=}10$)} \\
TET~\cite{deng2022tet}                   & 10 & 3.66\,M  & 83.17$^\ast$ \\
SpikFormer~\cite{zhou2023spikformer}     & 10 & 2.57\,M  & 80.90$^\ast$ \\
SWformer~\cite{fang2024swformer}         & 10 & 2.05\,M  & 82.90$^\ast$ \\
RTFormer~\cite{wang2024rtformer}         & 10 & ---      & 83.60$^\ast$ \\
QKFormer~\cite{zhou2024qkformer}         & 16 & 1.50\,M  & 84.00$^\ast$ \\
SPARTA~\cite{jang2025sparta}             & 16 & 13.80\,M & 83.06$^\ast$ \\
STAtten~\cite{lee2025statten}            & \rev{16} & 2.57\,M  & \rev{81.10}$^\ast$ \\
SpikePool~\cite{lee2025spikepool}        & 16 & 0.55\,M  & 80.20$^\ast$ \\
SpikePool~\cite{lee2025spikepool}        & 16 & 2.19\,M  & 82.70$^\ast$ \\
Mamba-Spike~\cite{li2024mambaspike}$^\ast$ & 10 & ---$^\S$ & \underline{92.50} \\
Mamba-Spike~\cite{li2024mambaspike}$^\ddagger$ & 10 & 36.25\,M & 48.90 \\
Mamba-Spike~\cite{li2024mambaspike}$^\dagger$ & 10 & 36.25\,M & 65.30 \\
\textbf{\name (ours)}                    & 10 & \textbf{0.87\,M} & \textbf{81.40} \\
\bottomrule
\end{tabular}
\end{minipage}\hfill
\begin{minipage}[t]{0.48\linewidth}\centering
\begin{tabular}{lccc}
\toprule
Method & $T$ & Params & Acc.\ (\%) \\
\midrule
\multicolumn{4}{l}{\emph{N-Caltech101} ($T{=}10$, 101 classes)} \\
SDT~\cite{yao2023sdt}                    & \rev{16} & 2.57\,M  & \rev{81.80}$^\P$ \\
STAtten~\cite{lee2025statten}            & \rev{16} & 2.57\,M  & \rev{83.15}$^\ast$ \\
SWformer~\cite{fang2024swformer}         & 10 & 2.05\,M  & \underline{88.45}$^\ast$ \\
SpikePool~\cite{lee2025spikepool}        & 10 & 0.55\,M  & 81.62$^\ast$ \\
SpikePool~\cite{lee2025spikepool}        & 10 & 2.19\,M  & 85.01$^\ast$ \\
\textbf{\name (ours)}                    & 10 & \textbf{0.88\,M} & \textbf{85.54} \\
\midrule
\multicolumn{4}{l}{\emph{DailyDVS-200} ($T{=}10$, 200 classes, $224{\times}224$)} \\
C3D~\cite{dailydvs200}          & -- & 147.2\,M & 21.99$^\ast$ \\
ESTF~\cite{dailydvs200}         & -- & 46.7\,M  & 24.68$^\ast$ \\
SDT~\cite{dailydvs200}          & -- & ---      & 35.43$^\ast$ \\
SpikFormer~\cite{dailydvs200}   & -- & ---      & 36.94$^\ast$ \\
TSM~\cite{dailydvs200}          & -- & 24.3\,M  & 40.87$^\ast$ \\
SlowFast~\cite{dailydvs200}     & -- & 33.6\,M  & 41.49$^\ast$ \\
TimeSformer~\cite{dailydvs200}  & -- & 121.2\,M & 44.25$^\ast$ \\
Swin-T~\cite{dailydvs200}       & -- & 27.8\,M  & 48.06$^\ast$ \\
EvMamba~\cite{evmamba}          & -- & 76.5\,M  & 49.65$^\ast$ \\
DarkShake-DVS~\cite{darkshakedvs} & -- & ---    & \underline{51.99}$^\ast$ \\
\textbf{\name (ours)}           & 10 & \textbf{8.04\,M} & \textbf{45.74} \\
\quad{}\emph{lite}              & 10 & 3.82\,M & 42.41 \\
\midrule
\multicolumn{4}{l}{\emph{N-MNIST} ($T{=}10$, saturated)} \\
PLIF~\cite{fang2021plif}                 & 10 & 0.87\,M  & \underline{99.61}$^\ast$ \\
SpikFormer~\cite{zhou2023spikformer}     & 10 & 2.57\,M  & \underline{99.61}$^\ast$ \\
Mamba-Spike~\cite{li2024mambaspike}$^\ddagger$ & 10 & 1.03\,M  & 99.36 \\
Mamba-Spike~\cite{li2024mambaspike}$^\dagger$ & 10 & 1.03\,M  & 99.49 \\
\textbf{\name (ours)}                    & 10 & \textbf{0.87\,M} & \textbf{99.57} \\
\bottomrule
\end{tabular}
\end{minipage}
\end{table}

\begin{figure}[t]
\centering
\vspace{-6pt}
\includegraphics[width=\linewidth]{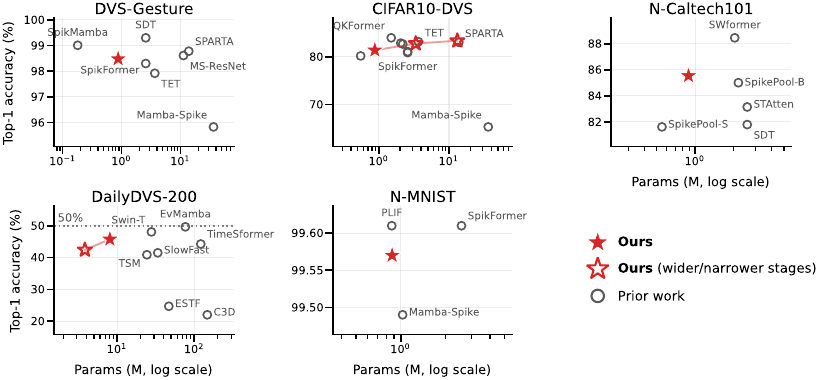}
\vspace{-6pt}
\caption{Accuracy vs.\ inference-time parameters (log scale). Panels plot
exactly the Table~\ref{tab:sota-all} entries reporting a parameter count and
label the frontier, so
the DailyDVS-200 panel's \emph{visible} ceiling is EvMamba ($49.65$) not the
benchmark best ($51.99$); its dotted line marks $50\%$ top-1. \name (filled
star) is Pareto-optimal on CIFAR10-DVS, N-Caltech101 and DailyDVS-200, but
not on DVS-Gesture, where SpikMamba is smaller and more accurate, nor on the
saturated N-MNIST. Hollow stars are rescaled variants: off-frontier on
CIFAR10-DVS, but on DailyDVS-200 the \emph{lite} config is itself a frontier
point.}
\vspace{-8pt}
\label{fig:pareto}
\end{figure}

\section{Ablations}\label{sec:ablation}
All ablations use the identical protocol, so each $\Delta$ characterizes the
deployed model directly. The DVS-Gesture $\Delta$ is against the five-seed
mean ($98.18$), never the best seed; the single-seed C10 and N-Caltech101
columns are differenced against Table~\ref{tab:sota-all}'s single-run
reference ($81.40$, $85.54$; five-seed means $81.28\pm0.11$, $85.49\pm0.18$).

% Tables 2 and 3 side by side in one top float (the former wraptables split
% across a page break once the 10-page revision moved the section).
\begin{table}[t]
\vspace{4pt}
\begin{minipage}[t]{0.48\textwidth}
\centering
\caption{Component-wise ablation and cross-dataset transfer. DVSG $\Delta$ uses the 5-seed mean; C10/NCal are single-seed.}
\label{tab:ablation-main}
\renewcommand{\arraystretch}{0.82}
\setlength{\tabcolsep}{1.6pt}
\scriptsize
\begin{tabular}{llccc}
\toprule
\# & Ablation & DVSG ($\Delta$) & $\Delta$ C10 & $\Delta$ NCal \\
\midrule
A1 & --\,DS-MPA $\to$ id.        & 97.73 (--0.45) & --0.80 & \textbf{--3.16} \\
A2 & --\,RepConv $\to$ $3{\times}3$ & 97.73 (--0.45) & --0.60 & --1.94 \\
A3 & --\,TDM ($\lambda{=}0$)     & 98.11 (--0.07) & --1.40 & --2.43 \\
A4 & --\,MR-S3A $\to$ st.-3 only & 98.11 (--0.07) & --1.10 & --1.46 \\
\midrule
\multicolumn{2}{l}{\textbf{\name} (full)} & $\mathbf{98.18\pm0.31}$ & \textbf{81.40} & \textbf{85.54} \\
\bottomrule
\end{tabular}
\end{minipage}\hfill
\begin{minipage}[t]{0.50\textwidth}
% Body of Table 3; placed side by side with Table 2 by sections/07_ablation.tex
\centering
\caption{\textbf{Parameter-matched \rev{replacements}} (mean$\,\pm\,$s.d.\rev{; sequence modules 5 seeds/cell, other components 3 seeds}). Each \rev{sequence} arm is matched to within $0.1\%$ of hierarchical-Mamba (None is $23.7\%$ smaller)\rev{; other rows to within $0.9\%$ of the full model (Appendix~\ref{app:controls})}. $^{***}p{<}0.001$, $^{*}p{<}0.05$, n.s.\ $p{\geq}0.05$.}
\label{tab:seq-ablation}
\renewcommand{\arraystretch}{0.85}
\setlength{\tabcolsep}{2pt}
\scriptsize
\begin{tabular}{lcc}
\toprule
Module & DVS-Gest.\ ($\Delta$) & N-Cal101 ($\Delta$) \\
\midrule
\textbf{Hier.\ Mamba (ours)} & $\mathbf{98.18\pm0.31}$ & $\mathbf{85.49\pm0.18}$ \\
Transformer   & $97.65\;(-0.53^{*})$ & $82.75\;(-2.75^{***})$ \\
GRU           & $97.73\;(-0.45^{*})$ & $82.48\;(-3.01^{***})$ \\
Conv1D        & $97.96\;(-0.22^{\textrm{n.s.}})$ & $82.24\;(-3.26^{***})$ \\
None          & $97.20\;(-0.98^{***})$ & $79.44\;(-6.05^{***})$ \\
\midrule
\multicolumn{3}{l}{\rev{\emph{Other components}}} \\
\rev{DS-MPA$\to$sigmoid} & \rev{---} & \rev{$83.92\;(-1.57^{***})$} \\
\rev{MR-S3A$\to$shared proj.} & \rev{---} & \rev{$84.00\;(-1.49^{***})$} \\
\rev{Spiking$\to$ConvGRU} & \rev{---} & \rev{$82.34\;(-3.15^{***})$} \\
\bottomrule
\end{tabular}

\end{minipage}
\end{table}

\paragraph{Component-wise (Table~\ref{tab:ablation-main}).}
Each headline component is removed and reverted to its closest fallback. On
DVS-Gesture the four effects span one to two misclassified clips. Anchored to
the five-seed mean $98.18\pm0.31$ rather than the best seed ($98.48$), the
removals are ${\approx}1.5\sigma$ for DS-MPA and RepConv and
${\approx}0.2\sigma$ for TDM and MR-S3A: none clears noise here. The
finer-grained benchmarks resolve all four\rev{: tested against the five-seed mean
with a prediction-interval $t$-test (equal seed variance assumed), every removal is significant on
CIFAR10-DVS and N-Caltech101 ($p\le0.016$, Holm-corrected; Appendix~\ref{app:sigma})}. DVS-Gesture is therefore
too saturated to resolve two of them. DS-MPA
dominates on N-Caltech101 ($-3.16$\,pp, the largest single effect anywhere),
and every $\Delta$ is largest on the hardest benchmark, so DVS-Gesture
understates their importance.

\paragraph{Sequence module, parameter-matched (Table~\ref{tab:seq-ablation}).}
Replacing the sequence mixer alone at a fixed budget (five seeds each),
DVS-Gesture again cannot separate the arms: a depthwise Conv1D is
indistinguishable from Mamba ($-0.22$\,pp, $p{=}0.22$). N-Caltech101
separates the arms decisively: Transformer $-2.75$, GRU $-3.01$, Conv1D
$-3.26$\,pp, all $p<0.001$, deleting the mixer entirely costs $-6.05$\,pp,
confirming the DVS-Gesture ordering was noise. A flat, budget-matched Mamba
reaches comparable accuracy on CIFAR10-DVS and N-Caltech101, so the hierarchy is an
efficiency choice. \rev{The same test applied to the other components gives the
same answer (Table~\ref{tab:seq-ablation}, bottom; N-Caltech101, Welch's $p\le3{\times}10^{-4}$
against the five-seed model). A sigmoid kernel with DS-MPA's projections loses
$1.57$\,pp; a shared concat-pool projection in place of MR-S3A loses $1.49$\,pp; a
ConvGRU front-end loses $3.15$\,pp at $2.4\times$ the estimated compute energy.
MR-S3A also holds in a model stripped of DS-MPA, RepConv and TDM ($+2.96$\,pp
over a stage-3 projection, $p{=}6{\times}10^{-5}$). With batch-norm statistics
frozen, MR-S3A folds exactly into one convolution, so its gain, like RepConv's,
comes from the training parameterization (Appendix~\ref{app:controls}).}

% Table 4 (moved from App. C.14 for the revision; referenced from Sec. 5).
% Placed after Tables 2-3 in source so Tables 1-3 keep their numbers.
\begin{table}[t]
\centering
\color{\revcolor}
\caption{\rev{\textbf{Weight storage against on-chip memory.} Weight memory per chip as
published~\citep{davies2018loihi,intel2021loihi2,nazeer2023spinnaker2lm,brainchip2022akida};
Akida supports at most 4-bit weights, so both models are sized at 4 bits there
and at 8 bits elsewhere. ``Chips'' is the minimum number of chips whose weight
memory holds the model.}}
\label{tab:chip-fit}
\renewcommand{\arraystretch}{1.08}
\setlength{\tabcolsep}{4pt}
\footnotesize
\begin{tabular}{lcc cc cc}
\toprule
& \multicolumn{2}{c}{On-chip weight memory} & \multicolumn{2}{c}{\name ($1.16$\,M)} & \multicolumn{2}{c}{Mamba-Spike ($36.25$\,M)} \\
\cmidrule(lr){2-3}\cmidrule(lr){4-5}\cmidrule(lr){6-7}
Processor & Size & Precision & Size & Chips & Size & Chips \\
\midrule
Loihi                 & $16$\,MB & 8\,bit & $1.2$\,MB & $\mathbf{1}$ & $36.3$\,MB & $3$ \\
Loihi 2               & $\le16$\,MB & 8\,bit & $1.2$\,MB & $\mathbf{1}$ & $36.3$\,MB & $3$ \\
SpiNNaker 2           & $19$\,MB & 8\,bit & $1.2$\,MB & $\mathbf{1}$ & $36.3$\,MB & $2$ \\
Akida AKD1000         & $8$\,MB  & 4\,bit & $0.6$\,MB & $\mathbf{1}$ & $18.1$\,MB & $3$ \\
\bottomrule
\end{tabular}
\end{table}

\paragraph{Sweeps (Appendix~\ref{app:exp}).} CSiLIF--SiLIF--SiLIF achieves the top accuracy among eight neuron mixes; accuracy plateaus at $T{=}16$; and performance is insensitive to $\lambda_{\text{L1}}$ across four orders of magnitude.

\section{Discussion and Conclusion}\label{sec:discussion}

Where the parameters go in a Spiking--Mamba hybrid is settled at the
interface: in the baseline we profile, $92\%$ of them sit in the one projection
joining front-end to scan, quadratic in resolution. A fixed-width
bridge holds the deployed count at $0.870$\,M across a $43\times$ area range.
That point covers the object and gesture benchmarks; DailyDVS-200 needs
width, and its $8.80$\,pp lead over spiking methods is at $8.04$\,M. \rev{A second hybrid
(SpikMamba) shows the same interface-only coupling at a $128\times$ smaller slope;}
swapping \rev{Mamba-Spike's} projection for a pooled one, nothing else changed, costs no
accuracy (Appendix~\ref{app:swap}). A front-end this small can carry
attention at all only because DS-MPA is bounded by construction
(Proposition~\ref{prop:bounded}), its normalizer unable to collapse the way a
saturating gate does, and removing it costs more than removing anything else
we ablate ($-3.16$\,pp on N-Caltech101).
Proposition~\ref{prop:ac-tdm} restores binary operands to all but the first
backbone convolution, and the commutation it rests on holds for any linear
temporal filter placed before a convolution.

% ===========================================================================
% ICLR 2027 policy additions, absent from the 2025 template this paper was
% previously built against. Both are placed here, between the main text and
% the bibliography, matching where the template itself places the analogous
% Ethics/Reproducibility statements ("at the end of the main text before
% references"). Per the template, the AI use statement is REQUIRED and does
% NOT count toward the 9-page main-text limit; it is a policy exemption, not
% something pdflatex enforces, so its physical page is not part of the
% page-budget accounting done elsewhere in this file.
% ===========================================================================
\subsection*{AI use statement}
Generative AI tools were used to critique the theoretical framing, to help
verify code and quantitative result checking, to identify literature and to
help compose and edit the manuscript. They were not used to fabricate or
alter experimental observations. Numerical claims were verified with the
original numerical records and experimental figures; mathematical claims
were re-derived and verified by the authors who take responsibility for the
final manuscript and released materials.

\subsection*{Reproducibility statement}
Splits, protocols and per-dataset hyperparameters are in
Section~\ref{sec:exp} and Appendix~\ref{app:hyper}; the backbone
specification and all proofs are in Appendices~\ref{app:mamba-spec}
and~\ref{app:proofs}. Headline accuracies are five-seed measurements, with
per-seed values in Appendix~\ref{app:seeds}. The resolution sweep behind
Fig.~\ref{fig:interface} is in Appendix~\ref{app:ressweep}. The training
code, covering all five benchmarks, is public at
\url{https://github.com/MuhiminOsim/RIPE-MambaSpike}.
It ships with the released Mamba-Spike code~\citep{mambaspike_code} and both
of the scripts behind its Table~\ref{tab:sota-all} rows, one running the
authors' own protocol ($^\ddagger$) and one running ours ($^\dagger$), so the
$48.90$ and $65.30$ CIFAR10-DVS figures can be reproduced directly rather than
taken on trust. The pooled-bridge control of Appendix~\ref{app:swap} is
included on the same footing.

\bibliographystyle{iclr2027_conference}
\bibliography{references}

\appendix
\section{Proofs of Propositions}\label{app:proofs}

\subsection{Proof of Proposition~\ref{prop:fusion} (Branch Fusion)}
We show that the seven-branch training-time RepConv of
Eq.~\eqref{eq:rep-train} is algebraically equivalent to the single
$3{\times}3$ convolution-plus-bias defined by Eqs.~\eqref{eq:fuse-W} and
\eqref{eq:fuse-b}. The argument has three parts: (i) BN absorption,
(ii) reduction of asymmetric and identity branches to $3{\times}3$
kernels, and (iii) composition of the sequential branch.

\paragraph{BN absorption.}
For any branch $b \in \mathcal B$ with weight $W_b$ and BatchNorm
parameters $(\gamma_b, \beta_b, \mu_b, \sigma_b^2, \varepsilon)$, the
affine action of BN on the convolution output is
\begin{align}
\mathrm{BN}_b(W_b * x)
&= \gamma_b \cdot \frac{W_b * x - \mu_b}{\sqrt{\sigma_b^2 + \varepsilon}}
   + \beta_b \nonumber\\
&= s_b\,(W_b * x) + (\beta_b - s_b\,\mu_b),
\end{align}
where $s_b = \gamma_b / \sqrt{\sigma_b^2 + \varepsilon}$ and we have used
the fact that running mean $\mu_b$ is a constant per output channel and
therefore acts as a bias. Each branch thus decomposes into a
\emph{linear} convolution with effective weight $s_b W_b$ and an additive
bias $\beta_b - s_b \mu_b$. Importantly, each internal branch batch-normalization
$\mathrm{BN}_b$ is a standard 2D BatchNorm evaluated across the combined batch and
temporal dimensions ($B \times T$), computing time-invariant channel statistics
$(\mu_b, \sigma_b^2)$ and affine parameters $(\gamma_b, \beta_b)$. Consequently,
the absorbed weights $W^\star$ and bias $b^\star$ are strictly static and time-independent,
yielding a single static $3\times3$ kernel for deployment. The temporal
batch-normalization (BNTT) is applied exclusively \emph{downstream} of the fused
convolution (i.e., $\mathrm{BNTT}_t(\mathrm{Conv}(W^\star, \cdot)_t)$) as an elementwise
affine transformation across timesteps, and does not interfere with the spatial kernel fusion.

\paragraph{Asymmetric and identity branches.}
The asymmetric kernels $1{\times}1$, $1{\times}3$ and $3{\times}1$ are
zero-padded to $3{\times}3$ support:
\begin{align}
\widetilde W^{1\times 1} &= \mathrm{Pad}_{[1,1,1,1]}(W^{1\times 1}),\\
\widetilde W^{1\times 3} &= \mathrm{Pad}_{[1,1,0,0]}(W^{1\times 3}),\\
\widetilde W^{3\times 1} &= \mathrm{Pad}_{[0,0,1,1]}(W^{3\times 1}),
\end{align}
where the padding tuples are $[\text{top},\text{bottom},\text{left},
\text{right}]$. These paddings preserve the convolutional action because
multiplication by zero is the identity for addition.

The identity branch is represented by the centered Kronecker delta kernel
$\mathbf I \in \RR^{c \times c \times 3 \times 3}$ with
$\mathbf I_{i,j,1,1} = \delta_{ij}$ and zero elsewhere. The AvgPool
branch is represented by the constant kernel
$\mathbf A_{i,j,p,q} = \tfrac{1}{9}\,\delta_{ij}$ for $p,q \in \{0,1,2\}$.

\paragraph{Sequential branch.}
The $1{\times}1{\rightarrow}3{\times}3$ sequential branch composes a
$1{\times}1$ projection with a $3{\times}3$ convolution. By associativity
of convolution, this is equivalent to a single $3{\times}3$ convolution
with weight
\begin{equation}
W^{\text{seq}}_{\text{fused}}
\;=\;
W^{3\times 3}_{\text{seq}} \;*\; \bigl(W^{1\times 1}_{\text{seq}}\bigr)^{\!\top},
\end{equation}
where the transpose swaps input and output channel axes for the
$1{\times}1$ kernel. Its BN is then absorbed as in the general case.

\paragraph{Summation.}
The seven branches act \emph{in parallel} on the same input $x$ and their
outputs are summed in Eq.~\eqref{eq:rep-train}. Because convolution is
distributive over addition,
\begin{equation}
\sum_b \widetilde W_b^{\text{eff}} * x + b_b^{\text{eff}}
\;=\;
\Bigl(\sum_b \widetilde W_b^{\text{eff}}\Bigr) * x + \sum_b b_b^{\text{eff}},
\end{equation}
which is exactly Eqs.~\eqref{eq:fuse-W}--\eqref{eq:fuse-b}. The equivalence
holds pointwise for every input $x$.\hfill$\square$

\subsection{Proof of Proposition~\ref{prop:tdm-spec} (TDM Spectral Shape)}
The transfer function of TDM is $H_c(z) = (1{+}\lambda_c) - \lambda_c\, z^{-1}$
(Eq.~\eqref{eq:tdm-z}). Evaluating on the unit circle $z = e^{j\omega}$,
\begin{align}
|H_c(e^{j\omega})|^2
&= \bigl((1{+}\lambda_c) - \lambda_c \cos\omega\bigr)^2
 + \bigl(\lambda_c \sin\omega\bigr)^2 \nonumber\\
&= (1{+}\lambda_c)^2 - 2\lambda_c(1{+}\lambda_c)\cos\omega
 + \lambda_c^2 \nonumber\\
&= 1 + 2\lambda_c(1{+}\lambda_c)(1 - \cos\omega).
\end{align}
Since $1 - \cos\omega \ge 0$ for $\omega \in [0,\pi]$ and is strictly
increasing on $[0,\pi]$, the squared magnitude is strictly increasing in
$\omega$ if and only if $\lambda_c(1{+}\lambda_c) > 0$, i.e.\ $\lambda_c \in
(-\infty, -1) \cup (0, \infty)$ (high-pass), and strictly decreasing if
and only if $\lambda_c(1{+}\lambda_c) < 0$, i.e.\ $\lambda_c \in (-1, 0)$
(low-pass). At $\lambda_c \in \{-1, 0\}$ the response is flat.\hfill$\square$

\subsection{Proof of Proposition~\ref{prop:bounded} (DS-MPA Boundedness)}
Let $\phi_i := \phi(Q_i) \in \RR^{C}_{\ge 0}$ and
$\psi_j := \phi(K_j) \in \RR^{C}_{\ge 0}$, and let
$w_{ij} := \phi_i^{\!\top}\psi_j \ge 0$. The attention numerator and
denominator for token $i$, channel $c$ are
\begin{equation}
n_{i,c} = \sum_{j=1}^N w_{ij}\, V_{j,c},
\qquad
d_i = \sum_{j=1}^N w_{ij}.
\end{equation}
We have $A_{i,c} = n_{i,c} / (d_i + \varepsilon)$. Define
$p_{ij} := w_{ij}/d_i \in [0,1]$ with $\sum_j p_{ij} = 1$ (a probability
distribution over tokens). Factoring out the normalizer yields the $\varepsilon$-exact form:
\begin{equation}
A_{i,c} \;=\; \frac{d_i}{d_i + \varepsilon} \sum_{j=1}^N p_{ij}\, V_{j,c}.
\end{equation}
Because $d_i \ge 0$ and $\varepsilon > 0$, the scaling factor satisfies $\eta_i := \frac{d_i}{d_i + \varepsilon} \in [0, 1)$. Hence $A_{i,c}$ is a convex combination of $\{V_{j,c}\}_{j=1}^N$ contracted by $\eta_i$, which lies strictly in the extended convex hull $\operatorname{Conv}(\{V_{j,c}\}_{j=1}^N \cup \{0\})$. In particular, for non-negative values $V_{j,c} \ge 0$,
\begin{equation}
0 \;\le\; A_{i,c} \;\le\; \max_j V_{j,c}.
\end{equation}
More generally, regardless of sign,
\begin{equation}
\min\Bigl(0,\, \min_j V_{j,c}\Bigr) \;\le\; A_{i,c} \;\le\; \max\Bigl(0,\, \max_j V_{j,c}\Bigr).
\end{equation}
No assumption on the sign or magnitude of $V$ is needed: the bound follows from $\phi \ge 0$ (making every $p_{ij} \ge 0$) together with the normalizer. If additionally $V_{j,c} \in [0,1]$, for instance if $V$ carries raw spikes rather than a linear projection of them, the convex hull is contained in $[0,1]$ and $A_{i,c} \in [0,1]$ as a special case. In all cases, $A_{i,c}$ is dynamically bounded and cannot saturate downstream activations.\hfill$\square$

\subsection{Proof of Proposition~\ref{prop:complexity} (DS-MPA Complexity)}
The dominant cost in Eq.~\eqref{eq:dsmpa-attn} is the product
$\phi(K)^\top V \in \RR^{C \times C}$, which requires $\mathcal{O}(N C^2)$
multiply-adds, followed by $\phi(Q) \cdot (\phi(K)^\top V)$ which is
$\mathcal{O}(N C^2)$ multiply-adds. The denominator
$\phi(K)^\top \mathbf{1}$ is $\mathcal{O}(N C)$ and
$\phi(Q) \cdot (\phi(K)^\top \mathbf{1})$ is $\mathcal{O}(N C)$, both
subdominant. The element-wise division and the modulation
Eq.~\eqref{eq:dsmpa-mod} are $\mathcal{O}(N C)$. Total:
$\mathcal{O}(N C^2)$ operations and $\mathcal{O}(NC + C^2)$ memory (the $C{\times}C$ key--value product plus the $N{\times}C$ output).

In comparison, standard softmax attention requires forming the
$N \times N$ score matrix at $\mathcal{O}(N^2 C)$ cost and $\mathcal{O}(N^2 + NC)$ storage, plus
$\mathcal{O}(N^2 C)$ for the value matmul. For our deepest stage with
$N = 256$ and $C = 128$, the ratio is
\begin{equation}
\frac{N^2 C}{N C^2} = \frac{N}{C} = \frac{256}{128} = 2,
\end{equation}
confirming the constant-factor speed-up cited in the body.\hfill$\square$

\subsection{Proof of Lemma~\ref{lem:tdm-grad} (TDM Gradient Decoupling)}
Let $\tilde x_t = (1+\lambda_c)\, x_t - \lambda_c\, x_{t-1}$ be the TDM
output, and let $f$ be a smooth downstream operator. By the chain rule,
\begin{equation}
\frac{\partial f(\tilde x_t)}{\partial x_{t-1}}
\;=\;
\frac{\partial f(\tilde x_t)}{\partial \tilde x_t}\,
\frac{\partial \tilde x_t}{\partial x_{t-1}}
+\frac{\partial f(\tilde x_t)}{\partial x_t^{(\text{spike})}}\,
\frac{\partial x_t^{(\text{spike})}}{\partial x_{t-1}},
\end{equation}
where the second term collects the gradient flowing through the spiking
state chain (involving the surrogate gradient
Eq.~\eqref{eq:atan}). The first term is
\begin{equation}
\frac{\partial f(\tilde x_t)}{\partial \tilde x_t} \cdot (-\lambda_c)
\;=\;
(-\lambda_c)\, f'(\tilde x_t),
\end{equation}
where $f'$ is the ordinary derivative of $f$ at $\tilde x_t$. This term
contains \emph{no surrogate gradient}, since TDM is a smooth linear
operator. The decomposition is therefore exactly the one claimed in
the lemma statement.\hfill$\square$

\subsection{Proof of Proposition~\ref{prop:ac-tdm} (Dual-Kernel Accumulate-Only Inference)}
Let $x_t\in\{0,1\}^{C_{\text{in}}\times H\times W}$ denote the binary spike tensor at timestep $t$. By the definition of TDM (\S\ref{sec:tdm}), for each input channel $c\in\{1,\dots,C_{\text{in}}\}$, the temporally modulated feature is
\begin{equation}
\tilde x_{t, c} \;=\; (1+\lambda_c)\, x_{t, c} - \lambda_c\, x_{t-1, c}.
\end{equation}
Let $W^\star\in\RR^{C_{\text{out}}\times C_{\text{in}}\times 3\times 3}$ and $b^\star\in\RR^{C_{\text{out}}}$ denote the fused spatial convolution kernel and bias (Proposition~\ref{prop:fusion}). The 2D spatial convolution at output channel $k\in\{1,\dots,C_{\text{out}}\}$ evaluates as:
\begin{equation}
\bigl[\mathrm{Conv}(W^\star, \tilde x_t)\bigr]_k \;=\; \sum_{c=1}^{C_{\text{in}}} W^\star_{k, c} * \tilde x_{t, c} + b^\star_k.
\end{equation}
Substituting the linear definition of $\tilde x_{t, c}$:
\begin{align}
\bigl[\mathrm{Conv}(W^\star, \tilde x_t)\bigr]_k
&= \sum_{c=1}^{C_{\text{in}}} W^\star_{k, c} * \bigl((1+\lambda_c)\, x_{t, c} - \lambda_c\, x_{t-1, c}\bigr) + b^\star_k \nonumber\\
&= \sum_{c=1}^{C_{\text{in}}} \bigl((1+\lambda_c)\, W^\star_{k, c}\bigr) * x_{t, c} - \sum_{c=1}^{C_{\text{in}}} \bigl(\lambda_c\, W^\star_{k, c}\bigr) * x_{t-1, c} + b^\star_k.
\end{align}
Defining the channel-scaled kernels $[W_a^\star]_{k, c} = (1+\lambda_c)\, W^\star_{k, c}$ and $[W_b^\star]_{k, c} = \lambda_c\, W^\star_{k, c}$, or in compact matrix notation:
\begin{equation}
W_a^\star \;=\; W^\star \cdot \diag(1+\lambda), \qquad W_b^\star \;=\; W^\star \cdot \diag(\lambda),
\end{equation}
we obtain the exact identity:
\begin{equation}
\mathrm{Conv}(W^\star, \mathrm{TDM}_\lambda(x))_t \;=\; W_a^\star * x_t - W_b^\star * x_{t-1} + b^\star.
\end{equation}
Because $x_t \in \{0, 1\}^{C_{\text{in}}\times H\times W}$ and $x_{t-1} \in \{0, 1\}^{C_{\text{in}}\times H\times W}$ are unquantized binary spikes, every spatial operation in $W_a^\star * x_t$ and $W_b^\star * x_{t-1}$ performs conditional addition (accumulation) over active event indices ($x_{t, c, u, v} = 1$) with zero multiplications. At $t=1$, boundary condition $x_0 = 0$ reduces the right-hand side to a single accumulate-only pass $W_a^\star * x_1 + b^\star$. The subtraction of the two pre-activations is an element-wise vector operation per timestep.\hfill$\square$

\section{Method Details}\label{app:method}
This appendix expands the modules of Section~\ref{sec:method} and states the two
propositions and the lemma whose proofs appear in Appendix~\ref{app:proofs}.

\subsection{Architectural Comparison with Mamba-Spike}
Table~\ref{tab:vs_mambaspike} summarizes the structural differences
between \name and Mamba-Spike along six design axes; the subsections
below expand each.
\begin{table}[h]
\centering
\caption{Architectural comparison of \name and Mamba-Spike.}
\label{tab:vs_mambaspike}
\renewcommand{\arraystretch}{1.15}
\begin{tabular}{lcc}
\toprule
Design axis & Mamba-Spike & \name \\
\midrule
Reparameterized convs       & \texttimes & \checkmark \\
Per-channel temporal filter & \texttimes & \checkmark (TDM) \\
Linear membrane attention   & \texttimes & \checkmark (DS-MPA) \\
Fixed-width SNN$\to$SSM bridge & \texttimes & \checkmark (MR-S3A) \\
Multi-resolution SNN bridge & \texttimes & \checkmark \\
Inference-fused $3{\times}3$ kernel & \texttimes & \checkmark \\
\bottomrule
\end{tabular}
\end{table}

\subsection{Mamba Discretization}
Given $\mathbf{u}\in\RR^{L\times d}$, Mamba evolves
$\dot{\mathbf h}(t)=A\mathbf h(t)+Bu(t)$, $y(t)=C\mathbf h(t)$, discretized with
step $\Delta$ by zero-order hold: $\bar A=\exp(\Delta A)$,
$\bar B=(\Delta A)^{-1}(\exp(\Delta A)-I)\Delta B$, giving
$\mathbf h_l=\bar A\mathbf h_{l-1}+\bar B u_l$, $y_l=C\mathbf h_l$, implemented as
a parallel selective scan at $\mathcal{O}(LNd)$. $B,C,\Delta$ are input-dependent.

\begin{figure}[h]
\centering
\includegraphics[width=\linewidth]{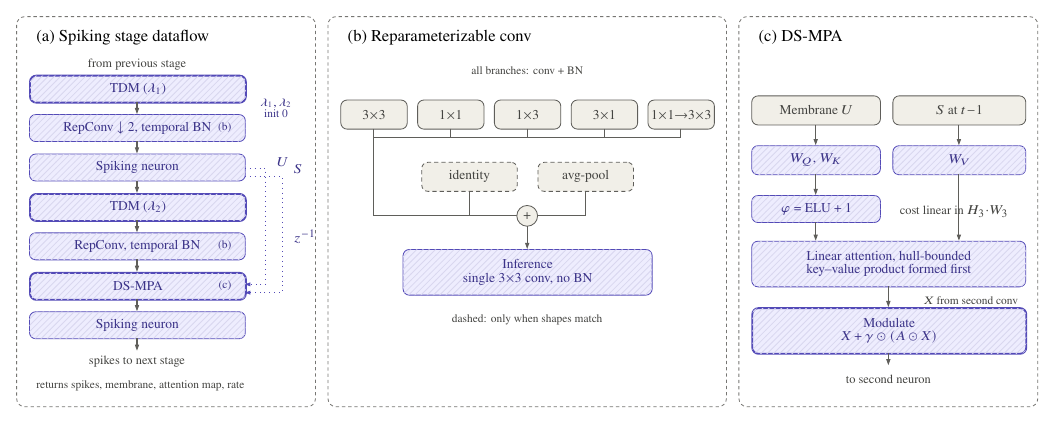}
\caption{RepSpikeStage internals. \textbf{(a)} Two
TDM$\to$RepConv$\to$BNTT$\to$neuron units. \textbf{(b)} The seven-branch RepConv fuses to a single
$3{\times}3$ kernel (Prop.~\ref{prop:fusion}). \textbf{(c)} DS-MPA, a dynamic convex-hull bounded
linear attention with $\phi{=}\mathrm{ELU}{+}1$, queries/keys from the
membrane $U$, values from the previous spike tensor.}
\label{fig:stage}
\end{figure}

\subsection{Reparameterization: Definitions and Parameter Algebra}
At training time each convolution is the seven-branch sum, each branch with its
own BatchNorm:
\begin{equation}\label{eq:rep-train}
\mathrm{RepConv}(x)=\sum_{b\in\mathcal{B}}\mathrm{BN}_b(W_b*x),\quad
\mathcal{B}=\{3{\times}3,1{\times}1,1{\times}3,3{\times}1,1{\times}1{\to}3{\times}3,\text{Id},\text{AvgPool}\}.
\end{equation}
With $s_b=\gamma_b/\sqrt{\sigma_b^2+\varepsilon}$ the fused kernel and bias are
\begin{align}
W^\star &= \textstyle\sum_{b\in\mathcal{B}}\mathrm{Pad}_{3\times3}(s_b W_b),\label{eq:fuse-W}\\
b^\star &= \textstyle\sum_{b\in\mathcal{B}}(\beta_b-\mu_b s_b).\label{eq:fuse-b}
\end{align}
The training-time count (six weight-bearing branches $+$ \emph{seven} BNs, since
the identity and average-pool branches carry a BN but no weights) is
$P^{\text{train}}=(9{+}1{+}3{+}3{+}9{+}1)c_{\text{in}}c_{\text{out}}+14c_{\text{out}}$,
fusing to $P^{\text{infer}}=9c_{\text{in}}c_{\text{out}}+c_{\text{out}}$, a ratio
$\to26/9\approx2.89$ as $c_{\text{in}}{=}c_{\text{out}}\to\infty$.

\paragraph{Execution order and stage-by-stage parameter accounting.}
The execution sequence in each spiking block is strictly:
\begin{equation}
x_t \xrightarrow{\quad\text{TDM}\quad} \tilde{x}_t \xrightarrow{\quad\text{RepConv}\quad} y_t \xrightarrow{\quad\text{BNTT}\quad} \hat{y}_t \xrightarrow{\quad\text{Neuron}\quad} s_t,
\end{equation}
where for the second block in each stage, DS-MPA modulates the post-BNTT features prior to the second neuron ($y_t \to \hat{y}_t \to \mathrm{DS\text{-}MPA}(U_1, S_{1,\text{prev}}, \hat{y}_t) \to s_t$). 
Within each $\mathrm{RepConvBlock}$, the seven branches possess internal 2D Batch Normalizations ($\mathrm{BN}_b$) computed across all batch and temporal items ($B \cdot T$). Because these internal BNs operate over the pooled temporal dimension, their running statistics and affine scales are time-invariant. Upon deployment, all seven branches and their internal BNs collapse into a single, static $3\times3$ convolutional weight $W^\star \in \RR^{c_{\text{out}} \times c_{\text{in}} \times 3 \times 3}$ and bias $b^\star \in \RR^{c_{\text{out}}}$ via Eqs.~\eqref{eq:fuse-W}--\eqref{eq:fuse-b}. Downstream of this fused convolution, BNTT applies an elementwise, per-timestep affine vector $(\gamma_t, \beta_t) \in \RR^{2 c_{\text{out}}}$ for $t \in [1, T]$ ($2 T c_{\text{out}}$ scalars in total). BNTT does \emph{not} replicate the spatial convolution kernels across timesteps.

We verify the exact stage-by-stage parameter breakdown by direct instantiation of our implementation:
\begin{itemize}[leftmargin=*,itemsep=1pt,topsep=2pt]
\item \textbf{Stage 1 ($c{=}32$):}
  $\mathrm{RepConv}_1$ ($2 \to 32$, stride 2: $10{,}624 \to 608$ fused), $\mathrm{RepConv}_2$ ($32 \to 32$, stride 1: $27{,}072 \to 9{,}248$ fused), $\mathrm{BNTT}_1 + \mathrm{BNTT}_2$ ($2 \times 2 \times T \times 32 = 1{,}280$ parameters at $T{=}10$), CSiLIF Neurons ($2 \times 5 \times 32 = 320$ parameters), $\text{DS-MPA}_1$ ($3 \times 32^2 + 32 = 3{,}104$), and TDM scalars $\lambda_1, \lambda_2$ ($2 + 32 = 34$). Stage 1 total: $42{,}434$ training $\to$ $14{,}594$ fused parameters ($T{=}10$).
\item \textbf{Stage 2 ($c{=}64$):}
  $\mathrm{RepConv}_1$ ($32 \to 64$, stride 2: $72{,}320 \to 18{,}496$ fused), $\mathrm{RepConv}_2$ ($64 \to 64$, stride 1: $107{,}392 \to 36{,}928$ fused), $\mathrm{BNTT}_1 + \mathrm{BNTT}_2$ ($2 \times 2 \times T \times 64 = 2{,}560$ parameters at $T{=}10$), SiLIF Neurons ($2 \times 3 \times 64 = 384$ parameters), $\text{DS-MPA}_2$ ($3 \times 64^2 + 64 = 12{,}352$), and TDM scalars ($32 + 64 = 96$). Stage 2 total: $195{,}104$ training $\to$ $70{,}816$ fused parameters ($T{=}10$).
\item \textbf{Stage 3 ($c{=}128$):}
  $\mathrm{RepConv}_1$ ($64 \to 128$, stride 2: $288{,}000 \to 73{,}856$ fused), $\mathrm{RepConv}_2$ ($128 \to 128$, stride 1: $427{,}776 \to 147{,}584$ fused), $\mathrm{BNTT}_1 + \mathrm{BNTT}_2$ ($2 \times 2 \times T \times 128 = 5{,}120$ parameters at $T{=}10$), SiLIF Neurons ($2 \times 3 \times 128 = 768$ parameters), $\text{DS-MPA}_3$ ($3 \times 128^2 + 128 = 49{,}280$), and TDM scalars ($64 + 128 = 192$). Stage 3 total: $771{,}136$ training $\to$ $276{,}800$ fused parameters ($T{=}10$).
\end{itemize}
Across the three stages, training parameters collapse from $1{,}008{,}674$ down to $362{,}210$ (a $2.78\times$ structural compression). All six fused convolutions sum to exactly $286{,}720$ parameters, which are strictly static and time-invariant. The BNTT layers across all three stages account for $8{,}960$ parameters at $T{=}10$ (or $14{,}336$ at $T{=}16$). Together with the entry temporal norm ($40$ params at $T{=}10$, $64$ at $T{=}16$), the entire fused spiking backbone comprises $362{,}250$ parameters (or $367{,}650$ at $T{=}16$). The MR-S3A bridge adds $133{,}411$: three per-scale projections ($18{,}496{+}36{,}928{+}73{,}792$), three softmax fusion weights, the bridge norm ($64$), and the attention-injection projection $W_A$ ($4{,}128$). The hierarchical Mamba stack adds $372{,}640$: three blocks at $361{,}760$, two transitions at $10{,}816$, and an input norm of $64$. With the classification head ($1{,}290$) and two frozen scalar coefficients retained from the token-ranking study of Table~\ref{tab:pruning-sweep}, this accounts for exactly $869{,}593$ parameters for a 10-class head at $T{=}10$ ($\approx 0.870$\,M). The head-independent core is $868{,}303$; at $T{=}16$ with an 11-class head the same model is $875{,}122$ (\S\ref{sec:complexity}).

The ATan surrogate used throughout is
\begin{equation}\label{eq:atan}
\frac{\partial s_t}{\partial u_t}\approx
\frac{\alpha}{2\bigl(1+(\tfrac{\pi}{2}\alpha(u_t-v_{\text{th}}))^2\bigr)}.
\end{equation}

\subsection{TDM: The Two Additional Views}
\paragraph{Spectral shape.}
TDM's transfer function is
\begin{equation}\label{eq:tdm-z}
H_c(z)=(1+\lambda_c)-\lambda_c z^{-1}.
\end{equation}
\begin{proposition}[TDM spectral shape]\label{prop:tdm-spec}
The squared magnitude response of $H_c$ is
$|H_c(e^{j\omega})|^2=1+2\lambda_c(1+\lambda_c)(1-\cos\omega)$, strictly
increasing in $\omega\in[0,\pi]$ (high-pass) for
$\lambda_c\in(-\infty,-1)\cup(0,\infty)$ and strictly decreasing (low-pass) for
$\lambda_c\in(-1,0)$.
\end{proposition}
\paragraph{Gradient decoupling.}
\begin{lemma}[TDM gradient decoupling]\label{lem:tdm-grad}
With TDM before a smooth operator $f$ and the LIF reset treated as a
stop-gradient,
$\partial f(\tilde x_t)/\partial x_{t-1}=(-\lambda_c)f'(\tilde x_t)+
(\partial f/\partial x_t^{(\text{spike})})(\partial x_t^{(\text{spike})}/\partial x_{t-1})$;
the first term carries no surrogate gradient.
\end{lemma}
The TDM path therefore carries a \emph{surrogate-independent} temporal
gradient. It survives regardless of whether the neuron is in its firing band,
since the first term contains no surrogate factor. We say
``surrogate-independent'' rather than ``unbiased'' deliberately: the lemma
establishes a differentiable path that bypasses the surrogate, not a
statistical property of a gradient estimator. \rev{The scope is also \emph{local}.
The surrogate-free term is TDM's local Jacobian: holding $x_t$ fixed,
$\partial\tilde x_t/\partial x_{t-1}=-\mathrm{diag}(\lambda)$, followed by an
ordinary convolution derivative. It is a direct temporal connection between
input tensors that does not pass through the recurrence of the neuron that
emitted $x_{t-1}$ and $x_t$, whose own path (the second term) requires its
surrogate. A loss gradient still reaches $\tilde x_t$ through the downstream
neurons' surrogates, and reaching the weights upstream of $x_{t-1}$ requires the
upstream surrogate. TDM does not make training surrogate-free end to end.}

\subsection{DS-MPA Complexity}
\begin{proposition}[DS-MPA complexity]\label{prop:complexity}
Eq.~\eqref{eq:dsmpa-attn} evaluates in $\mathcal{O}(NC^2)$ time and
$\mathcal{O}(NC+C^2)$ memory, versus $\mathcal{O}(N^2C)$ time and
$\mathcal{O}(N^2+NC)$ memory for softmax attention (time ratio $N/C$;
$2\times$ at the deepest stage, $N{=}256$, $C{=}128$).
\end{proposition}

\paragraph{Mathematical failure mode of sigmoid-gated attention vs.\ DS-MPA.}
A parameter-matched sigmoid attention gate is the natural alternative to
DS-MPA's normalized feature map, and it carries a specific failure mode:
\begin{equation}
A_{\text{sig}} = \sigma(Q K^\top) V,
\end{equation}
where $\sigma(z) = (1 + e^{-z})^{-1}$ and $Q, K, V$ share the identical projections as DS-MPA. In backpropagation, the gradient with respect to pre-activation $z = Q K^\top$ is:
\begin{equation}
\frac{\partial A_{\text{sig}}}{\partial z} = \sigma'(z) \cdot V = \sigma(z)\bigl(1 - \sigma(z)\bigr) \cdot V.
\end{equation}
When input event rates burst, membrane potentials $U$ fluctuate over large dynamic ranges, pushing $|z| \gg 0$ into the flat saturation tails where $\sigma'(z) \to 0$. In spiking neural networks, backward gradients must already traverse the surrogate gradient non-linearity $\frac{\partial s}{\partial u} = \frac{\alpha}{2(1 + (\frac{\pi}{2}\alpha u)^2)}$ of the spiking neurons. Cascading this surrogate factor with $\sigma'(z) \to 0$ can attenuate the gradient reaching the early layers; we test this below. 

In contrast, DS-MPA's kernel feature map $\phi(z) = \mathrm{ELU}(z) + 1 \ge 0$ possesses an asymmetrical, strictly non-vanishing derivative:
\begin{equation}
\phi'(z) = \begin{cases} 1, & z > 0, \\ e^z > 0, & z \le 0. \end{cases}
\end{equation}
Because $\phi'(z) \ge \min(1, e^z) > 0$ everywhere and never saturates for positive activations ($z > 0$), \rev{DS-MPA's kernel cannot enter such a tail, and its} linear normalizer $\sum_j \phi(K_j) + \varepsilon$ contracts the output into the dynamic convex hull of $V$ (Proposition~\ref{prop:bounded}).

\rev{\emph{Measured.} We trained the parameter-matched control on N-Caltech101, where removing DS-MPA costs most (A1). It uses the same $W_Q, W_K, W_V$ and $\gamma$, and replaces DS-MPA's kernel by $A=\sigma(QK^\top\!/\sqrt{C})\,V/N$; the $1/N$ plays the role of DS-MPA's normalizer and keeps the output on the scale of $V$, so the two arms differ only in the kernel. The control reaches $83.92\pm0.18$ (three seeds) against $85.49\pm0.18$ (five seeds), $-1.57$\,pp (Welch $p{=}2.0{\times}10^{-4}$), and recovers only $50\%$ of DS-MPA's gain over the identity of A1. Boundedness alone does not explain the gap: with weights $\sigma(z_{ij})/N\ge0$ summing to at most one, the control's output also lies in the hull of the values and zero. The operators differ in normalization and factorization. DS-MPA normalizes by total affinity, so each row's mass is close to one, whereas the control's row mass $N^{-1}\sum_j\sigma(z_{ij})$ varies with the query. DS-MPA's affinity is also factorized, giving $\mathcal{O}(NC^2)$ contractions against $\mathcal{O}(N^2C)$ for the pairwise sigmoid operator. DS-MPA is thus both more accurate and asymptotically cheaper at matched parameters. The gradient measurements are descriptive: averaged over training, the total pre-clip gradient norm is the same under both kernels ($1.040$ vs.\ $1.034$), while a slightly smaller share reaches stage~1 under the sigmoid kernel ($10.7\%$ vs.\ $11.8\%$). This is consistent with, but does not establish, the saturation argument above; the accuracy gap is the evidence.}

\subsection{Dual-Branch Training with the SGC Loss}
Every Mixup batch runs twice through the shared weights: a discrete pass (hard
spikes, ATan surrogate) supervised by $\mathcal{L}_{\text{TET}}$, and a
continuous pass (soft spikes $\sigmoid(\alpha(u{-}\theta))$, exact gradient). The
symmetric $\mathcal{L}_{\text{SGC}}=\mathrm{MSE}(z_d,z_c)$ couples them (a
mutual-consistency penalty, not distillation). At inference the continuous branch
is dropped, RepConv is fused, BN is recalibrated on clean data, and one
discrete pass is averaged over $T$ (Fig.~\ref{fig:dual-branch}).
The consistency gap falls from a peak of $\approx0.11$ to $\approx1.8\times10^{-3}$
over training.

\begin{figure}[h]
\centering
\includegraphics[width=\linewidth]{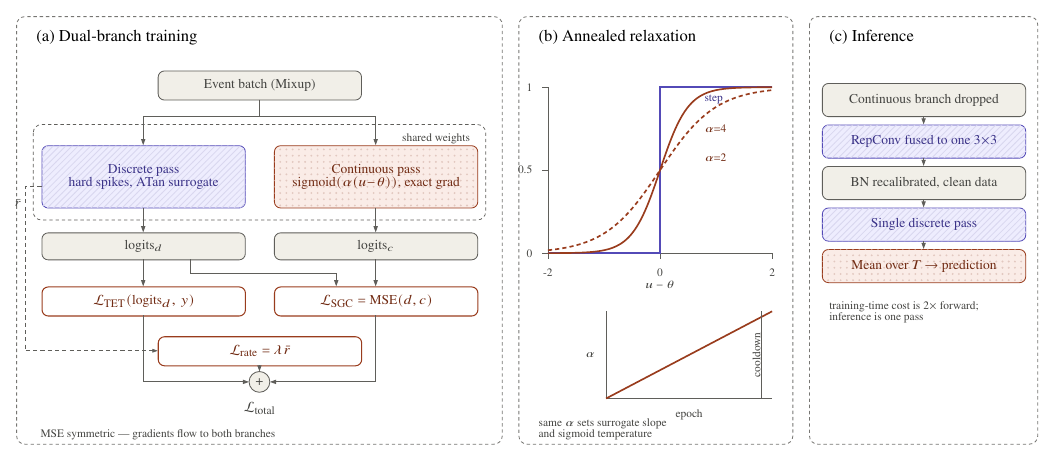}
\caption{Dual-branch training with the SGC loss. Coral marks training-time-only
components removed at deployment. During the continuous pass the BatchNorm
momentum is frozen so running statistics are not updated twice per step.}
\label{fig:dual-branch}
\end{figure}

\begin{table}[t]
\centering
\caption{Train-time vs.\ inference-time parameter count of \name
(hierarchical configuration, stage channels $[32,64,128]$, hierarchical
Mamba level dimensions $(32,64,128)$). Structural reparameterization fuses
the multi-branch spiking backbone into single $3{\times}3$ kernels, reducing backbone parameters by $\sim$$2.8\times$ at inference and yielding a
deployed model of \textbf{0.875\,M} parameters (Proposition~\ref{prop:fusion}).}
\label{tab:param-fusion}
\renewcommand{\arraystretch}{1.15}
\fittable{%
\begin{tabular}{lcc}
\toprule
Component & $P^{\text{train}}$ & $P^{\text{infer}}$ \\
\midrule
RepSpikeStage backbone (6 RepConv units) & 1.014\,M & 0.368\,M \\
Hierarchical Mamba + MR-S3A bridge + classifier & 0.508\,M & 0.508\,M \\
\midrule
\textbf{Full model}                      & \textbf{1.522\,M} & \textbf{0.875\,M} \\
\bottomrule
\end{tabular}}
\end{table}

\begin{figure}[t]
\centering
\small
\begin{verbatim}
# DS-MPA. U, S_prev, X : (B,T,C,H,W).  N = H*W tokens, C channels.
#   U      = membrane potentials at t      -> queries and keys
#   S_prev = BINARY spikes at t-1          -> values
#   X      = features to modulate
Q, K = W_Q(U), W_K(U)              # 1x1 convs; reshape to (B*T, N, C)
V    = W_V(S_prev)                 # 1x1 conv on spikes; unconstrained
phiQ, phiK = elu(Q) + 1, elu(K) + 1        # feature map, phi >= 0

KV  = phiK.transpose(1, 2) @ V     # (B*T, C, C)  contraction over TOKENS
num = phiQ @ KV                    # (B*T,N,C)  O(N C^2), no NxN matrix
den = (phiQ * phiK.sum(1, keepdim=True)).sum(-1, keepdim=True)
A   = num / (den + eps)            # convex comb. of V's rows -> Prop. 2
A   = A.permute(0, 2, 1).view(B, T, C, H, W)  # (B*T,N,C) -> (B,T,C,H,W)

return X + gamma * (A * X), A      # gamma: per-channel, init 1
\end{verbatim}
\caption{\textbf{DS-MPA in full.} Channel mixing, not token mixing: the
key--value contraction is over the $N$ token positions and produces a
$C\times C$ matrix, so cost is $\mathcal{O}(NC^2)$ and no $N\times N$ attention
matrix is ever materialized (contrast softmax attention's
$\mathcal{O}(N^2C)$). Queries and keys read the \emph{current} membrane
potential while values read the \emph{previous} binary spikes, and the
``dual-stream'' of the name. Two consequences are visible directly in the
code: $\phi\ge0$ plus the normalizer \texttt{den} make \texttt{A} a convex
combination of the rows of \texttt{V} (Prop.~\ref{prop:bounded}), and there is
no additive learnable gate bias anywhere, so no saturating scalar can zero the
gradient. \texttt{W\_V} is deliberately unconstrained; forcing
$V\in[0,1]$ would require a sigmoid and reintroduce exactly that failure mode.}
\label{fig:dsmpa-pseudo}
\end{figure}

\subsection{MR-S3A Bridge Detail}
Pseudo-code and the exact per-scale tensor shapes for the bridge introduced in
\S\ref{sec:mrs3a} are below. A learnable projection
$W_A\in\RR^{C_3\times d_{\text{m}}}$ additionally injects the stage-3 DS-MPA
attention map $\bar a_3$ into the token sequence via
$\mathbf{tokens}\mathrel{+}= W_A\bar a_3$, a resolution-independent
modulation (same $C_3{\to}d_{\text{m}}$ regardless of $H{\times}W$)
contributing $0.5\%$ of the deployed parameters.

\begin{figure}[t]
\centering
\small
\begin{verbatim}
# MR-S3A bridge. S_i: (B, T, C_i, H_i, W_i) spikes, i in {1,2,3}.
#   d_m: token width (fixed, independent of H x W)
#   N  = H_3 * W_3: token count at deepest resolution
tok = []                                        # collect per-scale
for i in 1, 2, 3:
    r     = cumsum(S_i, dim=T) / arange(1, T+1)  # cum. rate
    d     = r - shift(r, 1, dim=T)               # first diff.
    f_i   = concat([r, d], dim=C)                # 2*C_i chans
    f_i   = Pool_{H_3,W_3}(f_i)                  # H_3 x W_3
    tok_i = Proj_i(f_i)                          # 3x3 conv + BN -> d_m
    tok.append(tok_i.reshape(B, T, N, d_m))      # flatten spatial

w = softmax(w_hat)                               # in R^3
# Partition of unity across scales: (B, T, N, d_m)
tokens = w[0]*tok[0] + w[1]*tok[1] + w[2]*tok[2]
tokens = tokens + W_A(a3_pooled)                 # + attention mod.

return tokens  # all N tokens enter Mamba scan
\end{verbatim}
\caption{\textbf{MR-S3A in full.} Each stage's binary spike tensor is summarized
by a cumulative rate $r$ and its first difference $\delta$ (rate \emph{and}
instantaneous coding, $2C_i$ channels), pooled to the deepest stage's spatial
size $H_3{\times}W_3$, and projected to the shared token width $d_{\text{m}}$
by a per-scale $3{\times}3$ convolution with BatchNorm. The three scales are combined by a
learnable softmax weight $w\in\Delta^2$ (a partition of unity), not
concatenated or gated, so the output is always $N{=}H_3W_3$ tokens of width
$d_{\text{m}}$ regardless of the input resolution $H{\times}W$; only $H_1,
H_2, H_3$ (and $W_1,W_2,W_3$) change with resolution, and $\mathrm{Pool}$
absorbs that into a fixed output size before $d_{\text{m}}$ is ever touched.
This is the mechanism behind the resolution-independence of
Section~\ref{sec:mrs3a}: no weight in $\mathrm{Proj}_i$ or in the softmax gate
depends on $H$ or $W$. A learnable projection $W_A$ injects the stage-3 DS-MPA
attention map into the token sequence (matching the $+W_A\bar a_3$ of the code
and Fig.~\ref{fig:arch}). All $N$ tokens are retained, with no top-$k$ selection
or pruning is applied anywhere in this path. Exact per-level shapes
($d_{\text{m}}$, token count, pooling factors) are in
Table~\ref{tab:mamba-spec}.}
\label{fig:mrs3a-pseudo}
\end{figure}

\subsection{Formal Derivation of the CSiLIF Neuron}\label{app:csilif}
The Complex-valued Sigmoid LIF (CSiLIF) neuron used in Stage~1
(\S\ref{sec:repspike}) is a discrete-time resonate-and-fire
model~\citep{izhikevich2001resonate} with learnable complex poles, derived
below.

\paragraph{Continuous dynamics.}
Each channel $c$ maintains a complex-valued membrane state
$m(t)\in\CC$ evolving under the linear ODE
\begin{equation}\label{eq:csilif-cont}
\dot{m}(t) = \rho_c\, m(t) + b_c\, x(t),
\end{equation}
where $\rho_c = -e^{\ell_c} + j\,\omega_c \in \CC^-$ is a pole in the
open left half-plane (the parameterization $-e^{\ell_c}$ enforces
$\Real(\rho_c)<0$ for every $\ell_c\in\RR$) and $b_c\in\RR$ is a
per-channel input gain.

\paragraph{Zero-order hold discretization.}
With a learnable step $\Delta t_c = e^{\tau_c} > 0$
(parameterized as $\tau_c = \texttt{log\_dt}$ in the code), the exact
zero-order hold gives the discrete pole
\begin{equation}\label{eq:csilif-disc}
\alpha_c = \exp(\rho_c\,\Delta t_c)
        = \exp\!\bigl((-e^{\ell_c} + j\,\omega_c)\,e^{\tau_c}\bigr).
\end{equation}

\paragraph{Unconditional Schur stability.}
\begin{proposition}[CSiLIF stability]\label{prop:csilif-stable}
For every $(\ell_c,\omega_c,\tau_c)\in\RR^3$, the discrete pole satisfies
$|\alpha_c|<1$.
\end{proposition}
\begin{proof}
$|\alpha_c|
= |\exp(-e^{\ell_c}\Delta t_c)|\cdot|\exp(j\,\omega_c\,\Delta t_c)|
= \exp(-e^{\ell_c}\Delta t_c)\cdot 1
= \exp(-e^{\ell_c}\,e^{\tau_c})$.
Since $e^{\ell_c}>0$ and $e^{\tau_c}>0$, the exponent is strictly
negative, so $|\alpha_c|<1$. No clipping or projection is needed.
\end{proof}

\paragraph{Discrete recurrence with soft reset.}
The membrane evolves as
\begin{equation}\label{eq:csilif-recur}
m_t = \alpha_c\,(m_{t-1} - \kappa\, s_{t-1}) + b_c\, x_t,
\end{equation}
where $\kappa=0.5$ is the reset factor (a hyperparameter; $\kappa=1$
corresponds to hard reset). The real-valued readout and spike are
\begin{equation}\label{eq:csilif-spike}
U_t = 2\,\Real(m_t),\qquad
s_t = \Theta(U_t - v_{\text{th},c}),
\end{equation}
with $v_{\text{th},c}$ a learnable per-channel threshold and $\Theta$
replaced by the ATan surrogate (Eq.~\eqref{eq:atan}) during
backpropagation.

\paragraph{Interpretation.}
The imaginary part of $\rho_c$ gives each channel a learnable
oscillation frequency $\omega_c$, turning the low-pass response of a
real LIF into a damped band-pass. The readout $U_t=2\Real(m_t)$
discards the imaginary component, which carries phase information only;
the factor $2$ compensates for the energy split between real and
imaginary parts. Compared to SiLIF (Stages~2--3), which has
$\rho_c\in\RR^-$ and no oscillation, CSiLIF adds one learnable scalar
($\omega_c$) per channel at negligible parameter cost.

\section{Experimental Details and Extended Ablations}\label{app:exp}

This appendix supplies the material needed to reproduce every number in the
main text. That includes the exact per-dataset recipe (\S\ref{app:hyper}),
the backbone specification the parameter counts are derived from
(\S\ref{app:mamba-spec}), the four sweeps summarized in one paragraph of
\S\ref{sec:ablation} (\S\ref{app:sweeps}--\S\ref{app:tsweep}), the
optimization behaviour behind the reported schedule (\S\ref{app:dynamics}),
and qualitative evidence for the bounded-attention claim together with the
per-class breakdown behind the headline DVS-Gesture accuracy
(\S\ref{app:attn}). Unless stated otherwise every number is a single-seed
measurement under the protocol of \S\ref{sec:exp}; the five-seed results are
confined to the five headline accuracies.

\subsection{Instantiated Resolution Sweep}\label{app:ressweep}
The counts behind Fig.~\ref{fig:interface}, tabulated. Both models are
instantiated at each resolution and counted directly (no training, no
estimation); \name is counted after reparameterization fusion, i.e.\ the
deployed graph, and the configuration is held fixed down each column so this
isolates resolution alone.

\begin{table}[t]
\centering
\caption{\textbf{The SNN$\to$SSM interface's parameter cost scales with input
resolution; \name's does not.} \rev{All three} models instantiated and counted directly
at each resolution (\name after fusion, i.e.\ the deployed graph); configuration
held fixed down each column, so this isolates resolution alone. \rev{SpikMamba is its
official 16-frame configuration, with its unused spiking-patch module excluded;
its interface is the patch embedding plus positional and class embeddings.}}
\label{tab:resolution-scaling}
\renewcommand{\arraystretch}{1.0}
\footnotesize
\setlength{\tabcolsep}{4pt}
\begin{tabular}{lcccc>{\color{\revcolor}}c>{\color{\revcolor}}c>{\color{\revcolor}}c}
\toprule
Input & \name & \multicolumn{3}{c}{Mamba-Spike (flatten)} & \multicolumn{3}{c}{\rev{SpikMamba (patch)}} \\
\cmidrule(lr){2-2}\cmidrule(lr){3-5}\cmidrule(lr){6-8}
resolution & deployed (M) & total (M) & interface (M) & share & total (M) & interface (M) & share \\
\midrule
$34\times34$    & \textbf{0.870} & 4.79   & 2.10   & 43.8\% & 1.78 & 0.036 & 2.0\% \\
$48\times48$    & \textbf{0.870} & 7.41   & 4.72   & 63.6\% & 1.80 & 0.056 & 3.1\% \\
$128\times128$  & \textbf{0.870} & 36.25  & 33.55  & 92.6\% & 2.03 & 0.281 & 13.9\% \\
$224\times224$  & \textbf{0.870} & 105.46 & 102.76 & 97.4\% & 2.57 & 0.822 & 32.0\% \\
\midrule
growth         & $\mathbf{1.00\times}$ & $22.0\times$ & $49.0\times$ & --- & $1.44\times$ & $23.1\times$ & --- \\
\bottomrule
\end{tabular}
\end{table}

\rev{\emph{At native sensor resolution.} Both interfaces are linear in pixel count,
with slopes fixed by the instantiations above: $2048$ parameters per input pixel for
Mamba-Spike's flattened projection and $16$ for SpikMamba's per-patch positional
table (exact from $34^2$ up to $640^2$, where SpikMamba's interface reaches
$6.57$\,M, $79.0\%$ of its used parameters). At the native resolution of current
event sensors, the interface holds $98.5\%$ of Mamba-Spike's parameters and $45\%$
of SpikMamba's on a DAVIS346 ($346{\times}260$), $99.6\%$ and $74\%$ at
$640{\times}480$ (DVXplorer, Prophesee Gen3), and $99.9\%$ ($1.89$\,B) and $89\%$
at $1280{\times}720$ (Prophesee IMX636). \name has no parameter that depends on
$H{\times}W$; instantiated at $640{\times}480$ it deploys the same $869{,}593$
parameters as at $128^2$.}

\subsection{The Interface Swap in Isolation}\label{app:swap}
Table~\ref{tab:resolution-scaling} shows what the flatten-then-project bridge
costs in parameters. Whether it earns them is a separate question, and it is
answerable by changing that one layer in the \rev{official} Mamba-Spike model\rev{~\citep{mambaspike_code}}. In
place of flattening the front-end feature map and projecting it densely
(Eq.~\ref{eq:flatten-project}), we average-pool it to a fixed $8{\times}8$
grid and project from there. The spiking front-end, neuron models, Mamba
blocks and classifier are the \rev{official} ones, and both arms share a trainer,
optimizer, schedule and epoch budget; \texttt{input\_proj} is the only module
that differs.

\begin{table}[h]
\centering
\caption{Interface swap on CIFAR10-DVS at $128^2$, \rev{official} Mamba-Spike code
under its own protocol. Parameters are counted from the trained checkpoints.}
\label{tab:swap}
\renewcommand{\arraystretch}{1.1}
\begin{tabular}{lccc}
\toprule
Bridge & Total params & \texttt{input\_proj} & Top-1 \\
\midrule
Flatten-then-project (\rev{official}) & $36.25$\,M & $33.55$\,M ($92.6\%$) & $48.90$ \\
Pool to $8{\times}8$, then project & $\mathbf{4.80}$\,\textbf{M} & $\mathbf{2.10}$\,\textbf{M} ($43.7\%$) & $\mathbf{53.30}$ \\
\bottomrule
\end{tabular}
\end{table}

Pooling removes $16\times$ of the bridge and $7.6\times$ of the whole model,
and gains $4.40$\,pp on top (Table~\ref{tab:swap}). The measured
$92.6\%$ also confirms from a trained checkpoint the share quoted in
\S\ref{sec:intro} from instantiation alone. This protocol carries no
augmentation and drives both arms to $100\%$ train accuracy, so both sit well
below the augmented CIFAR10-DVS figures in Table~\ref{tab:sota-all}; the
comparison is internal to the pair, and it bounds what the resolution-coupled
projection buys rather than measuring either bridge at its best.

\subsection{Protocol Markers in the SOTA Table}\label{app:proto}

Table~\ref{tab:sota-all} carries four markers, expanded here so the caption
can stay short. $^\ast$ marks a figure quoted from the method's own paper. We
did not retrain those, so they are cross-protocol comparisons. $^\ddagger$
marks the \rev{official} Mamba-Spike code run under \emph{its own} protocol, and
$^\dagger$ the same code under \emph{ours}, which is the like-for-like
comparison. $^\S$ marks published Mamba-Spike figures that are for a larger
configuration whose size the paper does not state, so no parameter count is
attributable to them; the $36.25$\,M we quote is measured from the \rev{official}
code instead. Its published $99.4\%$ is on a MNIST variant it calls
Sequential MNIST~\citep{le2015irnn}, not N-MNIST, so we do not map the two
onto each other. On CIFAR10-DVS the published Mamba-Spike figure ($92.50$)
far exceeds our reproduction, and its \rev{published} baselines there likewise
exceed those methods' values elsewhere. That points to a different protocol,
so we rest the comparison on our controlled reproduction instead. The
\rev{$^\dagger$ run shares everything with our CIFAR10-DVS recipe except what is
model-specific: the same preprocessed event tensors, split, $128^2$ input and
$T{=}10$; the same augmentation (NDA, horizontal and polarity flip, Mixup $0.4$
with event CutMix, label smoothing $0.1$); the same batch ($8{\times}16$),
learning rate, warm-up and cosine schedule, EMA ($0.9999$) and $300$ epochs.
It differs only in weight decay ($0.03$ vs.\ $0.06$) and loss
(cross-entropy on its single output vs.\ our per-timestep TET). Its
accuracy and its $36.25$\,M parameters come from the same code, configuration
and checkpoint.}
DailyDVS-200 parameter counts are from~\citep{evmamba}, the only source
reporting them. Five-seed means for \name are in \S\ref{sec:exp}.

\subsection{Why Parameter Count Is the Binding Constraint}\label{app:footprint}

The efficiency claim is a deployed parameter count, not operations or energy (\S\ref{sec:discussion}). That axis is the one
that binds on the platforms event cameras are actually mounted on.

Event-based sensors are used predominantly where power and payload are
budgeted: micro-aerial vehicles, mobile robots, automotive perception and
wearable devices~\citep{gallego2020survey}. On such platforms weights are held
in on-chip SRAM or modest external flash, and on neuromorphic accelerators the
constraint is tighter still, since synaptic memory is distributed across cores
and is not expandable~\citep{merolla2014truenorth,davies2018loihi}. Unlike an
arithmetic budget, which trades against latency and can be met by running
slower, a weight budget binds \emph{absolutely}: a model whose parameters do not
fit does not run at reduced speed, it does not run at all.

Table~\ref{tab:footprint} makes the resulting gap concrete at half precision.

\begin{table}[h]
\centering
\caption{Deployed weight storage at FP16 (2\,bytes per parameter). Parameter
counts are the deployed/fused values used throughout; baseline counts are as
cited in the corresponding SOTA table. Storage is weights only; activations
and the runtime are additional.}
\label{tab:footprint}
\renewcommand{\arraystretch}{1.1}
\begin{tabular}{lcc}
\toprule
Model & Params & FP16 weights \\
\midrule
\textbf{\name} (small benchmarks) & $0.870$\,M & $\mathbf{1.7}$\,\textbf{MB} \\
\textbf{\name} \emph{lite} (DailyDVS-200) & $3.82$\,M & $\mathbf{7.6}$\,\textbf{MB} \\
\textbf{\name} (DailyDVS-200) & $8.04$\,M & $\mathbf{16.1}$\,\textbf{MB} \\
\midrule
TSM & $24.3$\,M & $48.6$\,MB \\
Swin-T & $27.8$\,M & $55.6$\,MB \\
Mamba-Spike & $36.25$\,M & $72.5$\,MB \\
EvMamba & $76.5$\,M & $153.0$\,MB \\
TimeSformer & $121.2$\,M & $242.4$\,MB \\
C3D & $147.2$\,M & $294.4$\,MB \\
\bottomrule
\end{tabular}
\end{table}

The practical reading is that the top of the DailyDVS-200 table is not
reachable on this class of device. Conceding $3.91$\,pp to EvMamba is a
different trade once EvMamba needs $153$\,MB of weight storage against our
$16.1$\,MB. The larger model is not a faster-or-slower option; it is an
unavailable one.

\rev{The same holds on neuromorphic silicon, where weights must sit in the chip's
synaptic memory. Table~\ref{tab:chip-fit} compares the full weight store of each
model, dual-kernel $W_b^\star$ included for ours, with the on-chip weight memory of
current neuromorphic processors, at the precision each chip supports. The raw
weight bytes of \name fit one chip's weight memory with over $90\%$ to spare;
Mamba-Spike's need the aggregate capacity of two to three chips at $128^2$, and because its interface grows quadratically with
resolution, ${\ge}7$ Loihi-class chips at $224^2$ and ${\ge}119$ at the
$1280{\times}720$ of an HD event sensor, where ours is unchanged. Raw weight
bytes are a lower bound for both models, since every platform adds indexing
overhead. The comparison concerns storage, the prerequisite for single-chip deployment:
\name meets it on every processor listed, Mamba-Spike on none. Quantized
accuracy, placement, operator support, state memory and the mapping of the
selective scan are the steps that follow for both models alike. For reference, the conventional
fused graph has $41\times$ fewer weights than Mamba-Spike, and the graph with
explicitly stored dual kernels $31\times$ fewer.}

\subsection{Per-Seed Values and Noise Floors}\label{app:seeds}

Five seeds per headline result. DVS-Gesture:
$\{97.73,98.11,98.11,98.48,98.48\}$, i.e.\ $98.18\pm0.31$, best $98.48$; on the
$264$-sample split one clip is $0.38$\,pp, so $\sigma\approx0.83$ clips.
CIFAR10-DVS: $\{81.1,81.3,81.3,81.3,81.4\}$, i.e.\ $81.28\pm0.11$, best
$81.40$, a much tighter $\sigma\approx1.1$ of its $1{,}000$ test samples.
N-Caltech101: $85.49\pm0.18$, and the reported $85.54$ is the median and modal
outcome ($3/5$); on $823$ test samples $\sigma\approx1.5$ samples. N-MNIST:
$99.55\pm0.02$, best $99.57$, i.e.\ $\sigma\approx1.8$ of its $10{,}000$ test
samples, smaller than the gaps separating the leading methods there, which is
why we treat that benchmark as saturated. DailyDVS-200:
$\{45.91,45.74,45.91,45.05,46.08\}$ top-1, i.e.\ $45.74\pm0.40$, and
$\{69.38,69.58,69.73,68.75,70.07\}$ top-5, i.e.\ $69.50\pm0.49$, on the
$4{,}093$-clip test split where one clip is $0.024$\,pp.

\subsection{Per-Dataset Hyperparameters}
\label{app:hyper}
Table~\ref{tab:hyper} gives the complete recipe per dataset as a set of
deltas against one default column, so that what actually varies is visible at
a glance. Three groups of differences are worth calling out, since each was
forced by a property of the data rather than tuned freely. The two harder,
smaller-sample benchmarks, CIFAR10-DVS and N-Caltech101, need more
regularization: drop-path doubles to $0.2$, weight decay drops to $0.03$,
Mixup strengthens to $\alpha{=}0.4$ at probability $0.5$, and warm-up
lengthens to 15 epochs. This is the standard response to a few-thousand-sample
event dataset, and we did not search it per-dataset beyond these values. The
auxiliary losses, SGC and the spike-rate $L_1$, are enabled only for
DVS-Gesture and N-MNIST; on CIFAR10-DVS and N-Caltech101 we found a TET-only
objective at least as good, so those runs use it. N-Caltech101 and DailyDVS-200
were trained on Google Colab at the default batch of $32$; N-MNIST is the one
departure, compensating for its small $34^2$ frames with batch 64.

\begin{table}[h]
\centering
\caption{Per-dataset hyperparameters for \name. ``---'' means identical to
``Default''.}
\label{tab:hyper}
\renewcommand{\arraystretch}{1.12}
\setlength{\tabcolsep}{3pt}
\footnotesize
\fittable{%
\begin{tabular}{lcccccc}
\toprule
Hyperparameter & Default & DVSG & C10-DVS & NCal101 & NMNIST & DDVS-200 \\
\midrule
\multicolumn{7}{l}{\emph{Architecture}} \\
Stage channels $[c_1,c_2,c_3]$ & $[32,64,128]$ & --- & --- & --- & --- & $[64,128,256]$\,\emph{(lite)}\,/\,$[96,192,384]$ \\
Hier.\ level dims            & $(32,64,128)$ & --- & --- & --- & --- & $(64,128,256)$\,\emph{(lite)}\,/\,$(96,192,384)$ \\
Hier.\ pool ratios           & $(1,2,2)$ & --- & --- & --- & --- & --- \\
$d_{\text{state}}$           & $16$   & ---   & ---   & ---   & --- & --- \\
Drop-path rate               & $0.1$  & ---   & $0.2$ & $0.2$ & --- & $0.2$ \\
Classifier head width        & ---    & ---   & ---   & ---   & --- & $1024$$^\dagger$ \\
\midrule
\multicolumn{7}{l}{\emph{Optimization}} \\
Optimizer / Init LR / Min LR & AdamW / $10^{-3}$ / $10^{-6}$ & --- & --- & --- & --- & AdamW/$5\!\times\!10^{-4}$/$10^{-6}$ \\
Weight decay (decay group)   & $0.05$ & ---   & $0.03$ & $0.03$ & --- & $0.02$ \\
LR warm-up / Epochs   & $10$ ep / $300$ & --- & $15$ ep/--- & $15$ ep/--- & --- & $2000$ it.$^\P$/$250$ \\
Batch size                   & $32$   & ---   & ---   & ---   & $64$ & --- \\
Grad-accum steps             & $1$    & ---   & ---   & ---   & $2$ & --- \\
Grad-clip $\|g\|_2$          & $1.0$  & ---   & ---   & ---   & --- & --- \\
Surrogate $\alpha$ (start$\to$end) & $2\to4$ & --- & $2\to3$ & $2\to3$ & --- & --- \\
\midrule
\multicolumn{7}{l}{\emph{Loss}} \\
TET $\lambda$                & $5\!\times\!10^{-3}$ & --- & $10^{-2}$ & $10^{-2}$ & --- & --- \\
SGC $\lambda_{\text{SGC}}$   & $1.0$  & ---   & $0$\,$^{\#}$ & $0$\,$^{\#}$ & --- & $0$\,$^\ddagger$ \\
L1 $\lambda_{\text{L1}}$     & $10^{-4}$ & --- & $0$\,$^{\#}$ & $0$\,$^{\#}$ & --- & --- \\
Mixup $\alpha$ / prob        & $0.2$ / $0.3$ & --- & $0.4$/$0.5$ & $0.4$/$0.5$ & $0.4$/$0.5$ & $0.2$/$0.5$ \\
Label smoothing              & $0.1$  & ---   & $0.05$& $0.05$& $0.05$ & --- \\
EventCutMix $\alpha$         & ---    & ---   & $0.4$ & $0.2$ & $0.2$ & --- \\
\midrule
\multicolumn{7}{l}{\emph{Data}} \\
$T$ (time bins)              & $16$   & ---   & $10$  & $10$  & $10$ & $10$ \\
Spatial $(H,W)$              & $128^2$ & --- & --- & --- & $34^2$ & $224^2$ \\
NDA $n_{\text{ops}}$ / mag   & $1$ / $0.3$ & --- & $2$/$0.5$ & $2$/$0.2$ & $2$/--- & n/a$^\S$ \\
NDA roll / cutout cap        & $0.25$ / $0.25$ & --- & --- & $0.15$/$0.15$ & $0.15$/$0.20$ & n/a$^\S$ \\
Polarity flip / H-flip / T-rev prob & $0.5$/---/--- & --- & ---/$0.5$/--- & --- & ---/$0.5$/$0.5$ & n/a$^\S$/$0.5$/n/a$^\S$ \\
\bottomrule
\end{tabular}}
\begin{flushleft}\footnotesize
$^\dagger$ DailyDVS-200 inserts a widening projection
(\texttt{Linear}$\to$\texttt{LayerNorm}$\to$\texttt{GELU}$\to$\texttt{Linear})
before the classifier: a $384$-dimensional penultimate feature cannot form
the optimal simplex for $200$ classes, so it is widened to $1024$ first. The
other four benchmarks use a single linear head.
$^\P$ DailyDVS-200's warm-up is specified in optimizer iterations (steps),
not epochs, in the training script; the other columns' warm-up is
epoch-based. $^\ddagger$ Only the SGC consistency term is disabled for
DailyDVS-200; spike-rate $L_1$ stays at its default rate, unlike C10-DVS and
N-Caltech101 (marked $^{\#}$), where both are disabled for a TET-only
objective. $^\S$ Not implemented in the DailyDVS-200 data pipeline (no NDA
transform, no polarity flip or temporal reversal); DailyDVS-200's own
horizontal-flip augmentation (prob.\ $0.5$) is unrelated to the NDA/reversal
pipeline used by the other four datasets.\end{flushleft}
\end{table}

\subsection{Hierarchical Mamba Backbone Specification}
\label{app:mamba-spec}
Table~\ref{tab:mamba-spec} specifies the hierarchical scan that the
complexity analysis of \S\ref{sec:complexity} and the parameter counts of
Table~\ref{tab:resolution-scaling} are computed from. The design principle is that each level balances sequence length against
representational width. Moving down a level halves the token count $L_\ell$ and
doubles the model dimension $d_\ell$: the state-space scan cost
$\mathcal{O}(L_\ell d_{\text{state}} d_\ell)$ stays exactly constant, while the
input/output projection cost $\mathcal{O}(L_\ell d_\ell^2)$ scales by only
$2\times$ per level (subquadratic compared to unpooled widening). Total cost
is therefore a \emph{sum} of three compact level costs rather than one large
monolithic scan, which is what allows a wide final Mamba dimension ($d_{\text{m}}{=}128$)
at a sequence length of only $64$ retained tokens by the deepest level. The
state dimension is held at $d_{\text{state}}{=}16$ and expansion at $2$
throughout; these are Mamba defaults and we did not tune them. The Level-1
token count is $H_3W_3$, the deepest spiking stage's spatial grid (\S\ref{sec:mrs3a}), and halves at each subsequent level by the pool
ratios; none of this is a function that is tuned per resolution, which is
precisely the property that makes the deployed parameter count
resolution-independent (\S\ref{sec:complexity}).

\begin{table}[h]
\centering
\caption{Per-level configuration of the hierarchical bidirectional Mamba
backbone (DVS-Gesture, $H_3{=}W_3{=}16$). Each level is a forward$+$backward
scan pair; ``tokens'' is the retained spatial-token count (preserved across
all $T$).}
\label{tab:mamba-spec}
\renewcommand{\arraystretch}{1.12}
\setlength{\tabcolsep}{6pt}
\begin{tabular}{lcccc}
\toprule
Level & Model dim $d_{\text{m}}$ & Tokens $k_\ell$ & State $d_{\text{state}}$ & Expansion \\
\midrule
1 & 32  & 256 & 16 & 2 \\
2 & 64  & 128 & 16 & 2 \\
3 & 128 & 64  & 16 & 2 \\
\bottomrule
\end{tabular}
\end{table}

\subsection{Per-Component Ablation Significance}\label{app:sigma}
\rev{Each ablated model is a single run, and the full model has five seeds. The
appropriate single-run test is therefore the prediction-interval
$t$-test of one new observation $x$ against the five-seed sample
($\bar m, s$): $t=(x-\bar m)/(s\sqrt{1+1/5})$, with $4$ degrees of freedom. It
assumes that the ablated model has the full model's seed variance. Every
$\Delta$ in Table~\ref{tab:sigma} is taken against the five-seed \emph{mean}, not
against the single-run reference of Table~\ref{tab:ablation-main}, which makes
the CIFAR10-DVS effects slightly smaller than quoted there. With Holm
correction over the eight CIFAR10-DVS and N-Caltech101 tests, all eight remain
significant at $0.05$. DVS-Gesture, at $264$ test clips and $\sigma{=}0.31$\,pp,
cannot separate any component, and we treat CIFAR10-DVS and N-Caltech101 as the
load-bearing evidence.}

\begin{table}[h]
\centering
\color{\revcolor}
\caption{\rev{Ablation significance against the five-seed mean (prediction-interval
$t$-test, df $=4$, two-sided). DVS-Gesture $98.18\pm0.31$, CIFAR10-DVS
$81.28\pm0.11$, N-Caltech101 $85.49\pm0.18$.}}
\label{tab:sigma}
\renewcommand{\arraystretch}{1.12}
\setlength{\tabcolsep}{4pt}
\footnotesize
\begin{tabular}{lccc}
\toprule
Removed & DVSG $\Delta$ ($p$) & CIFAR10-DVS $\Delta$ ($p$) & NCal101 $\Delta$ ($p$) \\
\midrule
DS-MPA  & $-0.45$ ($0.26$) & $-0.68$ ($0.005$)  & $-3.11$ ($9{\times}10^{-5}$) \\
RepConv & $-0.45$ ($0.26$) & $-0.48$ ($0.016$)  & $-1.89$ ($7{\times}10^{-4}$) \\
TDM     & $-0.07$ ($0.85$) & $-1.28$ ($4{\times}10^{-4}$) & $-2.38$ ($3{\times}10^{-4}$) \\
MR-S3A  & $-0.07$ ($0.85$) & $-0.98$ ($0.001$)  & $-1.41$ ($0.002$) \\
\bottomrule
\end{tabular}
\end{table}

\subsection{Neuron-Mix and Pruning Sweeps}
\label{app:sweeps}
Two design choices in the deployed model are reported here as swept rather
than asserted. For the \emph{neuron mix} (Table~\ref{tab:neuron-mix}), we
evaluated eight depth-by-depth assignments of CSiLIF and SiLIF across the
three spiking stages. The deployed CSiLIF--SiLIF--SiLIF ordering is the best
of the eight, and placing CSiLIF in both early stages is the most harmful
configuration. Spike-rate variance is comparable across all eight, so the
choice rests on accuracy alone and we make no stability claim for it. The
result is consistent with the intended reading of CSiLIF: a complex-valued
resonate-and-fire neuron is most useful where the input still carries raw
temporal structure, i.e.\ at stage 1.

Token pruning is the wrong lever for this architecture
(Table~\ref{tab:pruning-sweep}). An earlier version selected a top-$k$ subset
of spatial tokens before the scan; we removed it, and the sweep below is why.
On DVS-Gesture, accuracy is identical at
$98.48\%$ for $k\in\{64,128,256\}$, where $k{=}256$ is the full
$16{\times}16$ grid, i.e.\ no pruning at all, so pruning bought no accuracy.
It also bought no parameters, since the invariance in
Table~\ref{tab:resolution-scaling} holds with pruning disabled because Mamba
blocks are $\mathcal{O}(d_{\text{m}}^2)$ and independent of sequence length,
and it bought little compute, since an $8\times$ change in sequence length
moves total FLOPs by only $11\%$ (Table~\ref{tab:flops}). A component that
buys none of the three is better removed than defended, so the deployed model
scans every token MR-S3A produces. One honest limit of the sweep: the
DVS-Gesture test split has $0.38$\,pp granularity per clip, so the tied
entries are indistinguishable at this sample size rather than provably equal.

\begin{table}[t]
\centering
\caption{Heterogeneous neuron-mix ablation on DVS-Gesture. Each row
assigns a neuron family $\{$LIF, SiLIF, CSiLIF$\}$ to each of the three
RepSpikeStages. ``SR-var'' is the cross-epoch standard deviation of the
mean spike rate during the last 50 epochs (lower = more stable). All rows are
measured single-seed runs; the deployed configuration (\textbf{bold}) attains
the highest accuracy, while training is comparably stable across all mixes
(SR-var $\le 10^{-3}$). Rows sorted by accuracy.}
\label{tab:neuron-mix}
\renewcommand{\arraystretch}{1.15}
\setlength{\tabcolsep}{4pt}
\begin{tabular}{lccc}
\toprule
Configuration (Stage 1, 2, 3) & Acc.\ (\%) & Params (M) & SR-var \\
\midrule
\textbf{CSiLIF, SiLIF, SiLIF (ours)} & \textbf{98.48} & \textbf{0.88} & \textbf{0.0005} \\
\midrule
SiLIF, SiLIF, CSiLIF                & 98.11 & 0.88 & 0.0004 \\
SiLIF, CSiLIF, CSiLIF               & 98.11 & 0.88 & 0.0003 \\
LIF, SiLIF, SiLIF                   & 97.73 & 0.88 & 0.0003 \\
LIF, LIF, LIF                       & 97.73 & 0.88 & 0.0009 \\
SiLIF, SiLIF, SiLIF                 & 97.35 & 0.88 & 0.0002 \\
CSiLIF, CSiLIF, SiLIF               & 96.59 & 0.88 & 0.0005 \\
CSiLIF, CSiLIF, CSiLIF              & 96.21 & 0.88 & 0.0005 \\
\bottomrule
\end{tabular}
\end{table}

\begin{table}[t]
\centering
\caption{\textbf{Why the deployed model prunes no tokens.} An earlier version
of this architecture selected a top-$k$ subset of spatial tokens before the
scan; this sweep is why we removed it. $k$ is the number of spatial tokens
retained per timestep and the full sequence is $N{=}H_3W_3{=}256$, so
$k{=}256$ \emph{is} the unpruned model. Accuracy is identical for
$k\in\{64,128,256\}$: pruning to a quarter of the tokens bought no accuracy.
It also bought no parameters --- the resolution invariance of
Table~\ref{tab:resolution-scaling} holds with pruning disabled, since Mamba
blocks are $\mathcal{O}(d_{\text{m}}^2)$ and independent of sequence length ---
and little compute, an $8\times$ change in sequence length moving total FLOPs by
$11\%$ (Table~\ref{tab:flops}). The deployed model therefore scans every token.
Caveat: the DVS-Gesture split has $0.38$\,pp granularity per clip, so the tied
entries are indistinguishable at this sample size rather than provably equal;
$k{=}32$ costing $0.75$\,pp is two clips.}
\label{tab:pruning-sweep}
\renewcommand{\arraystretch}{1.15}
\setlength{\tabcolsep}{6pt}
\begin{tabular}{lcccc}
\toprule
$k$               & 32 & 64 & 128 & \textbf{256 (full, deployed)} \\
\midrule
Tokens / sample   & 512 & 1024 & 2048 & \textbf{4096} \\
Acc.\ (\%)        & 97.73 & 98.48 & 98.48 & \textbf{98.48} \\
$\Delta$ (pp)     & --0.75 & 0.00 & 0.00 & \textbf{---} \\
\bottomrule
\end{tabular}
\end{table}

\subsection{Time-Step and Spike-Rate Sweeps}
\label{app:tsweep}
Table~\ref{tab:T-sweep} sweeps the number of time bins. Accuracy rises
steeply to $T{=}16$ and is then flat ($98.48\%$ at both $16$ and $20$), while
measured compute (MACs, not SOPs, \S\ref{sec:complexity}) grows strictly
linearly in $T$. $T{=}16$ is the knee; everything beyond it
buys computation with no accuracy. This is the single most consequential
efficiency knob in the model. Halving $T$ to $8$ halves the operation count at
a cost of $1.89$\,pp, a trade a deployment might reasonably take, whereas
raising $T$ past $16$ is never worthwhile.

Figure~\ref{fig:lambda-sweep} sweeps the spike-rate $L_1$ coefficient across
four orders of magnitude. Both accuracy and the mean spike rate $\bar r$ are
nearly flat over the whole range, with the default $10^{-4}$ giving the best
accuracy at a low rate. This is a negative result worth reporting: sparsity
here is governed mainly by the architecture and the surrogate-gradient
schedule, not by the explicit penalty, so $\lambda_{\text{L1}}$ is not a
sensitive hyperparameter and should not be presented as one.

\begin{table}[h]
\centering
\caption{Sensitivity to time bins $T$ on DVS-Gesture. All measured single-seed.
MACs rather than SOPs because TDM makes backbone convolutions dense
(\S\ref{sec:complexity}).}
\label{tab:T-sweep}
\renewcommand{\arraystretch}{1.15}
\begin{tabular}{lccccc}
\toprule
$T$ & 4 & 8 & 12 & 16 & 20 \\
\midrule
Acc.\ (\%) & 94.32 & 96.59 & 97.73 & \textbf{98.48} & 98.48 \\
MACs (M)   & 430.5 & 861.0 & 1291.4 & 1721.9 & 2152.4 \\
\bottomrule
\end{tabular}
\end{table}

\begin{figure}[h]
\centering
\includegraphics[width=0.72\linewidth]{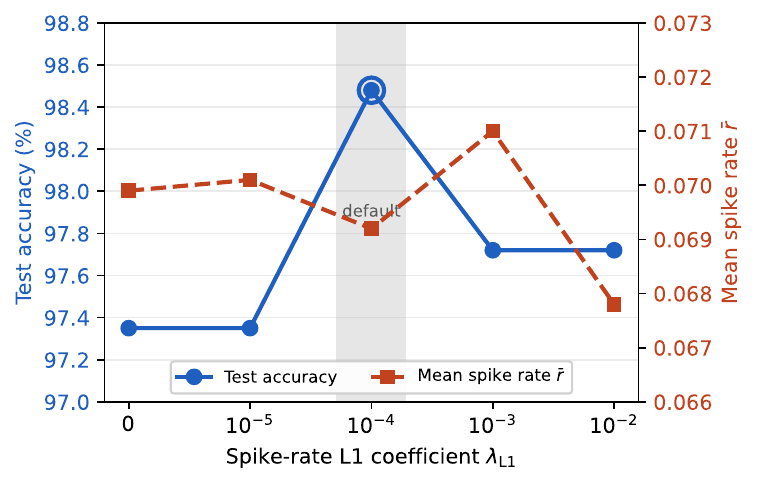}
\caption{Sensitivity to the spike-rate L1 coefficient $\lambda_{\text{L1}}$ on
DVS-Gesture. The default $10^{-4}$ gives the highest accuracy at a low spike
rate; accuracy and $\bar r$ are otherwise nearly flat across four orders of
magnitude.}
\label{fig:lambda-sweep}
\end{figure}

\subsection{Training Dynamics}
\label{app:dynamics}
Figure~\ref{fig:spike-rates} tracks per-stage spike rates during training.
The pattern is depth-dependent and stable: stage 1 settles near
$\bar r\approx0.2$ while deeper stages fall well below it. This matters for
two claims made elsewhere. It supports the firing-rate range assumed by the
accumulate-only accounting of Proposition~\ref{prop:ac-tdm}
(\S\ref{sec:complexity}). It also explains why the
spike-driven top-$k$ ranking was informative in the pruning ablation
(Table~\ref{tab:pruning-sweep}): with deep-stage activity this sparse, the
saliency signal cleanly separated active from inactive positions, even though
pruning to it bought no accuracy.

\begin{figure}[h]
\centering
\includegraphics[width=0.72\linewidth]{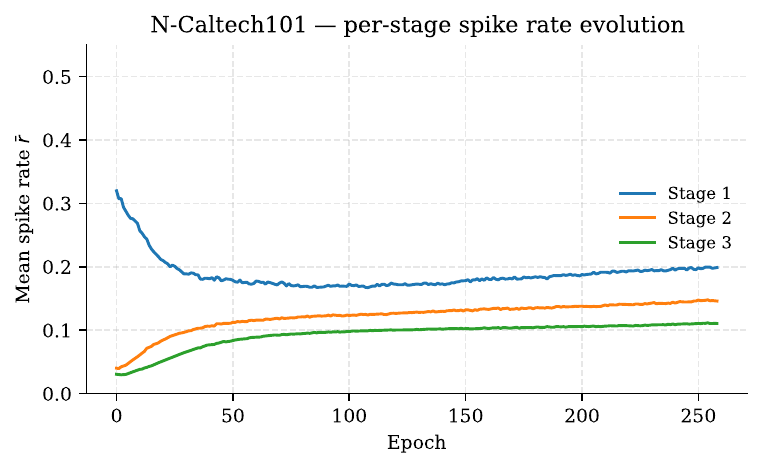}
\caption{Per-stage spike rate evolution (N-Caltech101; the same depth-dependent
pattern holds on DVS-Gesture): stage 1 stays near $\bar r\approx0.2$, deeper
stages settle well below it.}
\label{fig:spike-rates}
\end{figure}

\subsection{DS-MPA Attention and Per-Class Accuracy}
\label{app:attn}
Figure~\ref{fig:attn} visualizes stage-3 DS-MPA attention across four
timesteps of a DVS-Gesture ``hand-clap'' sample. Two things are visible. The
attention map is broadly distributed and peaks at the mid-gesture motion
maximum ($t{=}6$), then localizes as the hands meet ($t{=}14$): the module
tracks where evidence is, rather than collapsing onto a fixed region. Every value lies inside $[0,1]$ here, which is a property of this checkpoint's
value stream rather than a prediction of Prop.~\ref{prop:bounded}: with $V$ a
learned projection of prior spikes, the proposition bounds $A$ by the convex
hull of $V$ and the origin, not by $[0,1]$. The bound is constructive, so this
figure is a check that the implementation matches the proof rather than
evidence for the bound itself.

Table~\ref{tab:perclass} decomposes the headline $98.48\%$. Nine of the eleven
classes are perfect. The entire error budget sits in \emph{air-drums} and
\emph{air-guitar} at $91.67\%$ each, two misclassified clips apiece out of
twenty-four. These two are also the classes a human observer finds closest, both
sustained two-handed motions without the large directional sweep that
separates the arm gestures, so the residual error concentrates in the
semantically hardest pair rather than spreading across the task. With 264
test clips (24 per class), one clip is $0.38$\,pp, so these per-class figures should be read
as coarse: the difference between a perfect class and a $91.67\%$ class is just
two samples.

\begin{figure}[h]
\centering
\includegraphics[width=\linewidth]{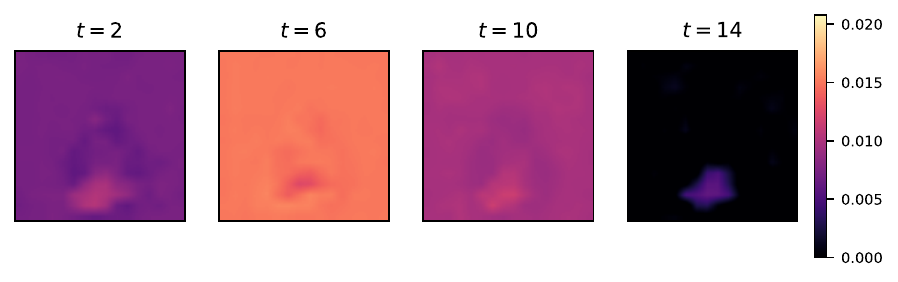}
\caption{Stage-3 DS-MPA attention $A_t$ (mean over channels) at
$t\in\{2,6,10,14\}$ on a DVS-Gesture ``hand-clap'' sample. Attention is broadly
distributed and strongest at the mid-gesture motion peak ($t{=}6$), then
localizes as the hands meet ($t{=}14$); values are dynamic convex-hull bounded
(Prop.~\ref{prop:bounded}, strictly within $[0,1]$ for unprojected binary spikes).}
\label{fig:attn}
\end{figure}

\begin{table}[h]
\centering
\caption{Per-class top-1 accuracy of \name on DVS-Gesture.}
\label{tab:perclass}
\renewcommand{\arraystretch}{1.12}
\setlength{\tabcolsep}{6pt}
\begin{tabular}{lc}
\toprule
Class & Acc.\ (\%) \\
\midrule
hand-clap, R/L-hand wave, R/L-arm cw, R/L-arm ccw, arm-roll, other & 100.00 \\
air-drums & 91.67 \\
air-guitar & 91.67 \\
\midrule
\textbf{Mean} & \textbf{98.48} \\
\bottomrule
\end{tabular}
\end{table}

\paragraph{Scope of the resolution-independence diagnosis.}
The coupling we remove is specific to flatten-then-project SNN--to--SSM
bridges. Spiking \emph{transformer} baselines are already flat in parameters
across input resolutions, so the mechanism we identify does not describe a
deficiency of attention architectures generally. It describes only the dense
projection that carries a spiking front-end into a state-space scan.

\paragraph{Accumulate-only accounting.}
A temporal operator placed \emph{before} a convolution silently makes that
convolution dense under standard sequential execution, since its operands are
no longer binary. The standard synaptic-operation methodology, applied
uncritically, would then credit our backbone with accumulate-only cost it
does not have in that execution mode. We flag this as a general accounting
caveat for any spiking backbone that uses pre-convolution temporal operators,
not something specific to our design; \S\ref{sec:complexity}
(Proposition~\ref{prop:ac-tdm}) gives an exact linear-commutation
reparameterization that recovers accumulate-only execution for neuromorphic
deployment, and we report an energy advantage only under that
reparameterization, not for the standard sequential path.

\paragraph{Full limitations.}
All five headline results are five-seed measurements ($98.18\pm0.31$ on
DVS-Gesture, $81.28\pm0.11$ on CIFAR10-DVS, $85.49\pm0.18$ on N-Caltech101,
$99.55\pm0.02$ on N-MNIST, $45.74\pm0.40$ on DailyDVS-200), as is the
parameter-matched sequence-module study of Table~\ref{tab:seq-ablation}. The
component-wise ablations remain single-seed, and we judge them against those
$\sigma$ rather than treating small $\Delta$ as real. On conventional
GPU/edge runtimes the model is parameter-efficient only, not
energy-efficient: executed sequentially, TDM precedes each backbone
convolution, so its operands are real-valued and accumulate-only accounting
does not apply, and the dense DS-MPA, bridge and Mamba stages dominate the
remaining cost. The accumulate-only property is recoverable for
neuromorphic deployment by the exact commutation of
Proposition~\ref{prop:ac-tdm}, which we verify algebraically and numerically
but have not measured on neuromorphic silicon (\S\ref{sec:complexity}). Weights shrink under fusion, but
activation memory grows under batching. The selective-scan kernel is
CUDA-specific, and quantized deployment on neuromorphic silicon is untested,
though the fused $3{\times}3$ backbone is directly amenable to it. Several
design choices are reported as deployed rather than swept: alternative DS-MPA
feature maps, per-dataset $k$, the grouped/low-rank DS-MPA variants of
Table~\ref{tab:q6-lowrank} (untrained), and the bridge width $d_{\text{m}}$
and pooling schedule. The dual-kernel accumulate-only path of
Proposition~\ref{prop:ac-tdm} is likewise proven and numerically verified
rather than benchmarked on hardware. DailyDVS-200 uses uniform event-to-frame binning
only, and sensitivity to that choice is unverified. The DailyDVS-200
lite/wide pair (\S\ref{sec:exp}) is an axis distinct from the
resolution-invariance claim of Fig.~\ref{fig:interface}: that claim fixes the
bridge width $d_{\text{m}}$ and varies $H{\times}W$, while the lite/wide pair
fixes $H{\times}W{=}224^2$ and deliberately varies $d_{\text{m}}$ as a
capacity choice.

\begin{table}[t]
\centering
\caption{\textbf{DS-MPA cost, isolated.} Measured at the exact
$(B,T,C,H,W)$ each DS-MPA instance receives in the deployed DVS-Gesture model
($B{=}4$ --- the batch the recipe trains at --- $T{=}16$, $128^2$ input), summed
over the three stages on one RTX 3050; CUDA-event timing, 30 reps after warm-up,
activation = peak allocated during one forward+backward. \emph{removed} returns
$X$ unchanged with $A{=}0$, so the downstream $A_t$ projection is inert.
Against a full training step of $405.8$\,ms, DS-MPA is $7.9\%$ of step time and
$4.3\%$ of parameters; substituting the cheap gate recovers only $5.5\%$ of step
time while forfeiting the $3.16$\,pp DS-MPA contributes on N-Caltech101
(Table~\ref{tab:ablation-main}) --- the largest single-component effect we measure.
Two properties make this more than a single-GPU anecdote. First, the
DS-MPA/gate \emph{ratio} is $3.30/3.33/3.31/3.28$ at $B{=}1/2/4/8$, stable to
$\pm0.8\%$ across an $8\times$ change in utilisation, and cost scales linearly
with $B$ from $B{=}1$, so the operator is compute-bound rather than
launch-bound: the ratio reflects arithmetic, not this device's idle time.
Absolute latencies are of course hardware-specific; the parameter and activation
figures are not. Second, the $O(NC^2)$ term does \emph{not} grow with channel
width here: $N$ falls $16\times$ as $C$ rises $4\times$, so $NC^2$ is constant
($4.19$\,M per stage) and measured forward time \emph{decreases} with depth.}
\label{tab:dsmpa-cost}
\renewcommand{\arraystretch}{0.95}
\setlength{\tabcolsep}{5pt}
\begin{tabular}{l rrrr}
\toprule
Channel-mixing operator & Params & Fwd (ms) & Fwd+bwd (ms) & Act.\ (MiB) \\
\midrule
\textbf{DS-MPA (ours)}             & $64{,}736$ & $11.48$ & $32.26$ & $790.6$ \\
$1{\times}1$ conv $+$ sigmoid gate & $21{,}728$ & $3.50$  & $9.75$  & $336.0$ \\
removed (identity)                 & $0$        & $0.29$  & $1.30$  & $112.0$ \\
\midrule
\multicolumn{5}{l}{\emph{does the $O(NC^2)$ term worsen with width?} DS-MPA at $B{=}1$:} \\
$[32,64,128]$   & $64{,}736$    & $3.11$  & $8.87$  & $197.7$ \\
$[64,128,256]$  & $258{,}496$   & $5.56$  & $16.12$ & $397.6$ \\
$[128,256,512]$ & $1{,}033{,}088$ & $12.21$ & $34.76$ & $805.4$ \\
\bottomrule
\end{tabular}
\end{table}

\subsection{Latency and Memory}\label{app:latency}
Table~\ref{tab:efficiency} reports CUDA-event latency and peak memory. Fusing
the multi-branch graph cuts $B{=}8$ latency $1.53\times$, beyond the
$1.74\times$ parameter cut. At $B{=}1$, the edge operating point, \name's
total footprint is below the baseline's, since the baseline's $138$\,MiB of
weights dominates there. That ordering reverses at $B{=}8$: \name's
activations are $\sim$$3.8\times$ larger, so it becomes slower and larger in
total memory under batching. The trade is weight memory for activation memory
and batched throughput.

\begin{table}[t]
\centering
\caption{Measured deployment cost on DVS-Gesture ($T{=}16$, $128{\times}128$,
fused; CUDA events, median of 30 runs after 10 warm-ups, RTX~3050). ``Weights''
is parameters $+$ buffers; ``act.'' is peak allocation minus weights. \name is
$41\times$ smaller in weights and its fused graph is $1.53\times$ faster than
its training graph, but its activations are larger than the baseline's --- the
trade is weight memory for activation memory and batched throughput.}
\label{tab:efficiency}
\renewcommand{\arraystretch}{1.1}
\setlength{\tabcolsep}{5pt}
\fittable{%
\begin{tabular}{lcccc}
\toprule
Model & Weights & $B{=}1$ lat. & $B{=}8$ lat. & $B{=}8$ peak / act. \\
\midrule
Mamba-Spike (baseline)    & 138.3\,MiB & 38.2\,ms & \textbf{51.5\,ms} & \textbf{456 / 272\,MiB} \\
\name, training graph     & 5.9\,MiB   & 36.5\,ms & 159.3\,ms & 1073 / 1043\,MiB \\
\textbf{\name, deployed}  & \textbf{3.4\,MiB} & \textbf{34.2\,ms} & 103.8\,ms & 1070 / 1043\,MiB \\
\bottomrule
\end{tabular}}
\end{table}

\begin{table}[t]
\centering
\caption{\textbf{Cheaper DS-MPA parameterizations are available and preserve
the bound, but we do not adopt them.} The bound is structural: $\phi\ge0$ and
the normalizer make $A_{i,\cdot}$ a convex combination of the rows of $V$
(Prop.~\ref{prop:bounded}), which holds for \emph{any} factorization of
$W_Q,W_K,W_V$ --- verified at every stage width by checking each $A_{i,c}$
against the componentwise range of $V_{\cdot,c}$ that variant produces
($9/9$). So the $O(C^2)$ projection cost \emph{can} be reduced without
theoretical loss. We nevertheless deploy the full operator, for a reason the
last column makes explicit: DS-MPA is only $7.4\%$ of the deployed
$875{,}122$ parameters, so even the $4\times$ reduction moves the headline
figure by $5.5\%$, while removing DS-MPA altogether costs $3.16$\,pp on
N-Caltech101 (Table~\ref{tab:ablation-main}) --- the largest single-component effect
we measure. We are not willing to risk a measured $3.16$\,pp for an unmeasured
$5.5\%$ of weight storage, and we have not trained these variants, so we make
no accuracy claim for them. They are reported as a sound design option, not as
a result.}
\label{tab:q6-lowrank}
\renewcommand{\arraystretch}{0.95}
\setlength{\tabcolsep}{5pt}
\begin{tabular}{l r r c r}
\toprule
DS-MPA parameterization & Params & vs.\ dep. & Bound & Deployed total \\
\midrule
\textbf{deployed}: $3\times(C\!\to\!C)$  & $64{,}736$ & $1.00\times$ & \checkmark & $\mathbf{875{,}122}$ \\
low-rank: $3\times(C\!\to\!\tfrac{C}{4}\!\to\!C)$ & $32{,}480$ & $0.50\times$ & \checkmark & $842{,}866$ ($-3.7\%$) \\
grouped: $3\times(C\!\to\!C)$, $g{=}4$   & $16{,}352$ & $0.25\times$ & \checkmark & $826{,}738$ ($-5.5\%$) \\
\bottomrule
\end{tabular}
\end{table}

\begin{table}[t]
\centering
\caption{\textbf{What resolution-independence does and does not buy.}
Complements Table~\ref{tab:resolution-scaling}: deployed parameters are
\emph{constant} across a $43\times$ range in pixel count, but activation memory
and latency are not --- they scale with resolution as in any convolutional
model. Measured on one RTX 3050, $T{=}16$, batch $4$, \name's DVS-Gesture
configuration; parameters counted after fusion. $^{\dagger}$At $224^2$ the
training step reaches this card's $8$\,GiB capacity, so the training throughput
there is allocator-limited rather than a clean measurement (inference, which
needs no gradients, still runs at $18.2$\,samples/s). This is the trade the
paper claims and no more: weight storage is decoupled from resolution, total
memory is not.}
\label{tab:q7-resolution-cost}
\renewcommand{\arraystretch}{0.95}
\setlength{\tabcolsep}{6pt}
\begin{tabular}{l r r r r}
\toprule
Input resolution & Deployed (M) & Train (samp/s) & Infer (samp/s) & Peak (GiB) \\
\midrule
$34\times34$   & \textbf{0.875} & $25.4$            & $102.0$ & $0.27$ \\
$48\times48$   & \textbf{0.875} & $25.4$            & $83.4$  & $0.45$ \\
$128\times128$ & \textbf{0.875} & $10.3$            & $50.1$  & $2.67$ \\
$224\times224$ & \textbf{0.875} & $0.8^{\dagger}$   & $18.2$  & $8.00^{\dagger}$ \\
\bottomrule
\end{tabular}
\end{table}

\begin{table}[t]
\centering
\caption{\textbf{Compute cost and how it scales with $k$, $T$, $C$ and
resolution.} FLOPs are \emph{measured} with PyTorch's flop counter on the fused
inference graph (torch's convention: one multiply--accumulate is $2$ FLOPs;
MACs $=$ FLOPs$/2$), so the selective scan's matmuls and the DS-MPA contraction
are counted as executed rather than derived from a formula. Batch $1$; the
deployed DVS-Gesture point is $[32,64,128]$, $T{=}16$, $128^2$, all tokens scanned.
Three things are worth reading off this table. \emph{First}, an $8\times$ change
in the SSM sequence length ($32\!\to\!256$ tokens per timestep) costs only
$11\%$ more compute, because the convolutional front-end dominates and not the
scan --- so the scan's length is not the lever that matters here. \emph{Second},
cost is exactly linear in $T$ and exactly quadratic in $C$ ($4\times$ per
width doubling). \emph{Third}, deployed \emph{parameters} are flat across
resolution while compute is not: this is the same trade
Table~\ref{tab:q7-resolution-cost} makes for memory. The small parameter
variation down the $T$ block is BNTT, whose affine terms are per-timestep.
This is a compute count and \emph{not} an energy claim: TDM precedes each
backbone convolution, so accumulate-only accounting does not apply
(\S\ref{sec:complexity}), and arithmetic counts say nothing about memory
traffic.}
\label{tab:flops}
\renewcommand{\arraystretch}{0.95}
\setlength{\tabcolsep}{5pt}
\begin{tabular}{l r r r}
\toprule
Configuration & Params (M) & GFLOPs & GMACs \\
\midrule
\textbf{deployed} ($T{=}16$, $128^2$) & $\mathbf{0.875}$ & $\mathbf{9.00}$ & $\mathbf{4.50}$ \\
\midrule
\multicolumn{4}{l}{\emph{SSM sequence length} (tokens entering the scan; $T{=}16$, $128^2$)} \\
$L/T{=}32$ & $0.875$ & $8.11$  & $4.05$ \\
$L/T{=}128$& $0.875$ & $8.49$  & $4.25$ \\
$L/T{=}256$& $0.875$ & $9.00$  & $4.50$ \\
\midrule
\multicolumn{4}{l}{\emph{time bins $T$} ($128^2$)} \\
$T{=}4$    & $0.864$ & $2.06$  & $1.03$ \\
$T{=}8$    & $0.868$ & $4.12$  & $2.06$ \\
$T{=}32$   & $0.890$ & $16.47$ & $8.23$ \\
\midrule
\multicolumn{4}{l}{\emph{stage widths $C$} ($T{=}16$, $128^2$)} \\
$[64,128,256]$   & $3.361$  & $32.73$  & $16.36$ \\
$[128,256,512]$  & $13.165$ & $130.49$ & $65.25$ \\
\midrule
\multicolumn{4}{l}{\emph{input resolution} ($T{=}16$) --- params flat, compute not} \\
$34\times34$     & $0.875$ & $0.78$  & $0.39$ \\
$48\times48$     & $0.875$ & $1.27$  & $0.63$ \\
$224\times224$   & $0.875$ & $24.69$ & $12.34$ \\
\bottomrule
\end{tabular}
\end{table}

\subsection{\texorpdfstring{\rev{Whole-Model Operations and Energy}}{Whole-Model Operations and Energy}}\label{app:energy}
\rev{Section~\ref{sec:complexity} costs only the AC-eligible backbone.
Table~\ref{tab:energy-whole} measures the whole model with the community
standard, NeuroBench~\citep{yik2025neurobench}, against a trained Mamba-Spike
under identical metrics, and breaks our cost down by component. We have no
neuromorphic \emph{energy} measurement; the energy figures are standardized
operation counts and first-order estimates.}

\rev{\emph{Our compute is dominated by the dense components, not the spikes.} DS-MPA,
the bridge and the Mamba stack are $89\%$ of compute energy and the dual-kernel
spiking backbone $10\%$ ($4.8$--$11.8\%$ for $\bar r$ between $0.05$ and
$0.20$). Executing the backbone as Prop.~\ref{prop:ac-tdm} prescribes cuts total
compute $2.3\times$ against the same graph all-MAC.}

\rev{\emph{The spiking front-end is why the backbone is cheap.} Replacing the three
spiking stages by ConvGRU stages at a matched total parameter count gives a
front-end of $11.1$\,mJ, $21\times$ our spiking backbone, and a whole model of
$13.0$\,mJ, $2.4\times$ ours despite having no DS-MPA: a gated RNN multiplies
real-valued tensors at every step and has no accumulate-only form. On
conventional hardware, where spikes run as dense MACs, the two whole models cost
about the same ($12.48$ vs.\ $12.97$\,mJ, the all-MAC rows of
Table~\ref{tab:energy-whole}), so there the spiking front-end's advantage is its
$3.15$\,pp higher accuracy at a matched parameter count (App.~\ref{app:controls}).}

\rev{\emph{Per inference, Mamba-Spike computes less.} Its trained front-end fires
very sparsely ($98.5\%$ activation sparsity), so its $33.55$\,M-weight projection
does little effective work, and NeuroBench measures $0.27$\,mJ against our
$6.25$\,mJ. \name is not an operation-light model, and we make no
per-inference energy claim over Mamba-Spike. What \name changes is storage:
$31\times$ fewer bytes, which decides whether a model can be held on a single
neuromorphic chip at all (Table~\ref{tab:chip-fit}). That is a feasibility
question, not a cost trade-off.}

\rev{\emph{The dual kernel costs storage and a one-frame buffer.} Prop.~\ref{prop:ac-tdm}
replaces each fused kernel $W^\star$ by $W_a^\star=W^\star\diag(1+\lambda)$, applied to
the current spikes $x_t$, and $W_b^\star=W^\star\diag(\lambda)$, applied to the previous
spikes $x_{t-1}$. Storing $W_b^\star$ beside $W_a^\star$ for the five AC convolutions adds $285{,}696$ weights ($+32\%$, to
$1.16$\,M), and the previous binary input frame must be held for each of them:
$425{,}984$ bits ($52$\,KiB) per sample, reused at every step. Both are included
in every count above.}

\rev{\emph{The dual-kernel form is verified in the full network.} On all $264$
DVS-Gesture test clips, executing the five AC-eligible convolutions as
Prop.~\ref{prop:ac-tdm} prescribes reproduces the fused graph's predictions on
every clip (maximum logit difference $0.02$ at logit scale $4.0$, from fp32
rounding near spike thresholds), and NeuroBench independently classifies all
ten dual-kernel convolutions as accumulate-only from their operands.}

\rev{\emph{The dual-kernel layer executes on neuromorphic hardware.} We ran stage-3
conv2 of the trained model on SpiNNaker (SpiNNaker~1, sPyNNaker~7.4) exactly as
Prop.~\ref{prop:ac-tdm} prescribes (Table~\ref{tab:spinnaker}). The model's own
binary spike tensor $S_1$, recorded on four DVS-Gesture test clips, is replayed by
spike sources. $W_a^\star$ and $W_b^\star$ are two sets of synapses, the latter reaching
$x_{t-1}$ through one additional window of synaptic delay. Every synaptic event is
the accumulation of one stored weight triggered by a binary spike. Non-spiking
integrators read out the summed synaptic input per model timestep, inverting their
leak exactly. The hardware input matches the PyTorch layer with $r{=}0.9999998$
(relative RMSE $5.8{\times}10^{-4}$), the residual being consistent with
SpiNNaker's 16-bit fixed-point weights, with no lost packets or saturation. The run
was deliberately slowed and spread in time to avoid packet loss, so it validates
the computation, not latency or energy. It covers 2 of the layer's 128 output
channels, with all 128 input channels.}

\begin{table}[h]
\centering
\color{\revcolor}
\caption{\rev{\textbf{Whole-model operations and energy.} \emph{Top:} NeuroBench~v2.3
algorithm-track metrics, measured on the $264$ DVS-Gesture test clips
($T{=}16$, $128^2$) with trained checkpoints. Ours is the deployed graph
(fused, all tokens, Prop.~\ref{prop:ac-tdm} dual-kernel convolutions as real
layers); NeuroBench itself classifies each layer as AC or MAC from its actual
operands. NeuroBench counts linear and convolutional layers only, so the
functional dense operations it cannot see (DS-MPA contractions and the
selective-scan cores) are counted analytically in their own row. \emph{Bottom:}
our per-component dense-equivalent breakdown (N-Caltech101 point, $T{=}10$,
measured presynaptic rates $0.20/0.14/0.10$), and the same model with its
spiking stages replaced by a parameter-matched ConvGRU front-end. Energy uses
the 45\,nm constants of \S\ref{sec:complexity} ($E_{\mathrm{MAC}}{=}4.6$,
$E_{\mathrm{AC}}{=}0.9$\,pJ); element-wise neuron updates are omitted for all
models, and memory access is not costed.}}
\label{tab:energy-whole}
\renewcommand{\arraystretch}{1.05}
\setlength{\tabcolsep}{5pt}
\footnotesize
\begin{tabular}{lrr}
\toprule
\multicolumn{3}{l}{\emph{NeuroBench, DVS-Gesture test set, per sample}} \\
Metric & \name & Mamba-Spike \\
\midrule
Parameters (incl.\ dual-kernel $W_b^\star$)  & $\mathbf{1.16}$\,\textbf{M} & $36.25$\,M \\
Footprint (fp32)                       & $\mathbf{4.7}$\,\textbf{MB} & $145.0$\,MB \\
Activation sparsity                    & $89.4\%$ & $98.5\%$ \\
Effective MACs                         & $803.5$\,M & $48.8$\,M \\
Effective ACs                          & $360.5$\,M & $38.5$\,M \\
Functional dense MACs (not in NeuroBench) & $485.5$\,M & $2.4$\,M \\
Compute energy                         & $6.25$\,mJ & $0.27$\,mJ \\
\bottomrule
\end{tabular}

\vspace{6pt}
\begin{tabular}{lrrrr}
\toprule
\multicolumn{5}{l}{\emph{\name, per-component breakdown ($T{=}10$, $128^2$)}} \\
Component & MAC (M) & AC (M) & Energy ($\mu$J) & Share \\
\midrule
Spiking backbone (dual-kernel AC)   & $23.6$   & $466.8$ & $528.6$  & $9.8\%$ \\
DS-MPA                              & $507.9$  & $18.7$  & $2353.2$ & $43.6\%$ \\
MR-S3A bridge                       & $330.3$  & ---     & $1519.4$ & $28.2\%$ \\
Hierarchical Mamba (incl.\ scan)    & $205.1$  & ---     & $943.6$  & $17.5\%$ \\
$A_t$ injection, pos.\ enc., head   & $10.6$   & ---     & $48.8$   & $0.9\%$ \\
\midrule
\textbf{Total}                      & $1077.6$ & $485.5$ & $\mathbf{5393.7}$ & \\
\quad same graph, all-MAC (GPU)     & $2713.3$ & ---     & $12481.3$ & \\
ConvGRU front-end (matched params)  & $2820.2$ & ---     & $12972.8$ & \\
\quad of which the ConvGRU stages   & $2417.2$ & ---     & $11118.9$ & $85.7\%$ \\
\bottomrule
\end{tabular}
\end{table}

\subsection{\texorpdfstring{\rev{Parameter-Matched Controls}}{Parameter-Matched Controls}}\label{app:controls}
\rev{The component ablations of Table~\ref{tab:ablation-main} remove a component; they
do not ask whether a simpler component of the same size would do. Each control in
Table~\ref{tab:controls} replaces one component by its natural alternative at a
matched parameter count, with everything else (backbone, bridge, Mamba, head,
recipe, data and evaluation) held fixed, on N-Caltech101, where the ablation
effects are largest. Each control loses clearly to the full model.}

\rev{\emph{Kernel.} A sigmoid gate with DS-MPA's projections and a $1/N$
normaliser (App.~\ref{app:method}) recovers only half of DS-MPA's gain over no
attention.}

\rev{\emph{Interface.} Pooling all three stages to the stage-3 grid, concatenating them
and applying one shared projection is no better than using stage~3 alone
(A4, $-1.41$\,pp against the five-seed mean). In this test, the multi-scale
information helps only with MR-S3A's per-scale parameterization. The two bridges compute the
same class of function at inference: with batch-norm statistics frozen, MR-S3A
folds exactly into one convolution over the concatenated scales,
$\sum_s a_s\mathrm{BN}_s(W_s{*}x_s)=\mathrm{Conv}\big(\mathrm{concat}_s x_s;\,
\mathrm{concat}_s(a_sg_sW_s),\,\sum_s a_s(\beta_s-g_s\mu_s)\big)$ with
$g_s=\gamma_s/\sqrt{v_s+\epsilon}$ (verified on trained weights to $2{\times}10^{-14}$ in
fp64; $129{,}056$ folded parameters). The $1.49$\,pp is therefore a gain of the
\emph{training} parameterization, as with RepConv. A plausible reason is scale
conditioning: in a trained model the per-scale projected variances before
normalization are $0.76$, $0.11$ and $0.10$, which MR-S3A equalizes before fusing
them with near-even learned weights ($0.31$, $0.31$, $0.38$); this is a diagnostic,
not an established cause.}

\rev{\emph{Front-end.} Three ConvGRU stages, with widths solved so that the whole model
matches the deployed parameter count, are $3.15$\,pp less accurate and, under the
accounting of App.~\ref{app:energy}, have an estimated $21\times$ higher front-end
and $2.4\times$ higher total compute energy. Under these settings the spiking
front-end is better on both axes.}

\rev{\emph{MR-S3A in isolation.} Removing every other RIPE component (DS-MPA,
RepConv and TDM, by the patches of A1--A3) and varying only the interface
isolates MR-S3A: it adds $2.96$\,pp over a stage-3-only projection
($p{=}6{\times}10^{-5}$), more than twice its $1.41$\,pp effect in the full
model: the measured interface effect is larger in the stripped setting. The stripped model with
MR-S3A is within $1.77$\,pp of the full model. The stage-3-only arm freezes the
other two projections, as A4 does, so this comparison is not
parameter-matched; the shared-projection row above is.}

\begin{table}[h]
\centering
\color{\revcolor}
\caption{\rev{\textbf{Controls on N-Caltech101.} \emph{Top:} parameter-matched
controls; each row replaces one component and keeps everything else fixed, and
$\Delta$ and Welch's $p$ are against the five-seed full model. \emph{Bottom:}
MR-S3A in isolation, in a model stripped of DS-MPA, RepConv and TDM; $\Delta$ is
MR-S3A against a stage-3-only projection.}}
\label{tab:controls}
\renewcommand{\arraystretch}{1.1}
\setlength{\tabcolsep}{4pt}
\footnotesize
\begin{tabular}{llccc}
\toprule
Replaced & Control & Params & Top-1 (seeds) & $\Delta$ ($p$) \\
\midrule
--- & \textbf{\name (full)} & $0.881$\,M & $\mathbf{85.49\pm0.18}$ (5) & --- \\
\midrule
DS-MPA kernel & sigmoid gate $\sigma(QK^\top\!/\sqrt{C})V/N$ & $=$ & $83.92\pm0.18$ (3) & $-1.57$ ($2{\times}10^{-4}$) \\
MR-S3A interface & shared concat-pool projection & $-131$ & $84.00\pm0.19$ (3) & $-1.49$ ($3{\times}10^{-4}$) \\
Spiking front-end & ConvGRU stages & $+0.9\%$ & $82.34\pm0.14$ (3) & $-3.15$ ($6{\times}10^{-7}$) \\
\midrule
\multicolumn{5}{l}{\emph{Stripped model (no DS-MPA, RepConv, TDM)}} \\
Interface & \textbf{MR-S3A} & $0.817$\,M & $\mathbf{83.72\pm0.21}$ (3) & $+2.96$ ($6{\times}10^{-5}$) \\
Interface & stage-3 projection only & $0.761$\,M & $80.76\pm0.19$ (3) & --- \\
\bottomrule
\end{tabular}
\end{table}

\begin{table}[h]
\centering
\color{\revcolor}
\caption{\rev{\textbf{Prop.~\ref{prop:ac-tdm} on SpiNNaker.} Stage-3 conv2 of the
trained model executed as two accumulate-only synaptic projections, replaying the
model's recorded spikes (App.~\ref{app:energy}). Reference: the same layer in PyTorch.}}
\label{tab:spinnaker}
\renewcommand{\arraystretch}{1.08}
\setlength{\tabcolsep}{5pt}
\footnotesize
\begin{tabular}{ll}
\toprule
Layer / inputs & stage-3 conv2 ($2$ of $128$ out-channels), $16{\times}16$; 4 clips, $T{=}16$ \\
Synapses / input spikes & $1{,}083{,}392$ ($W_a^\star$ and $W_b^\star$) / $48{,}932$ \\
Synaptic events (all accumulate-only) & $1{,}708{,}272$ \\
Resources & $64$ cores on $5$ chips \\
Lost packets / saturation & $0$ / $0$ \\
Agreement with PyTorch & $r{=}0.99999985$; relative RMSE $5.8{\times}10^{-4}$; gain $1.0001$ \\
Per clip & $r\in[0.99999983,\,0.99999987]$ \\
\bottomrule
\end{tabular}
\end{table}

\end{document}